\documentclass{article}
 \usepackage{amsmath} 
     \PassOptionsToPackage{numbers, sort&compress}{natbib}

\usepackage[main, final]{neurips_2026}
\usepackage{tikz} \usetikzlibrary{positioning,arrows.meta,fit,backgrounds,calc,shapes.multipart,shadows}

\usepackage{amsthm}
\newtheorem{proposition}{Proposition}

\usepackage[utf8]{inputenc} 
\usepackage[T1]{fontenc}    
\usepackage{hyperref}       
\usepackage{url}            
\usepackage{booktabs}       
\usepackage{amsfonts}       
\usepackage{nicefrac}       
\usepackage{microtype}      
\usepackage{xcolor}         
\usepackage{booktabs, multirow, graphicx}
\title{Concorde: Geometry-Aware Link Prediction via Decoupled Energy Minimization}

\author{Ege Demirci  \\
Department of Computer Science, UC Santa Barbara  \\ 
Santa Barbara, CA 93106, USA \\
\AND
Ambuj Singh  \\
Department of Computer Science, UC Santa Barbara  \\ 
Santa Barbara, CA 93106, USA \\
}

\begin{document}

\maketitle

\begin{abstract}
Link prediction is difficult in sparse graphs, where diffusion
heuristics lack connectivity and geometric methods suffer spectral
instability when coupling learned representations with diffusion
operators. We introduce \textbf{Concorde}, an energy-based framework
that resolves this by decomposing link likelihood into three
structurally decoupled branches: \textbf{(i)} a parameter-free spectral anchor that establishes permutation-invariant topological stability; \textbf{(ii)} a mixed-curvature residual on a learnable
$\mathbb{H}_{c_h}^{d_h} \times \mathbb{S}_{c_s}^{d_s}$ product
manifold, regularized by a feature-aware Ollivier--Ricci curvature
proxy; and \textbf{(iii)} a continuous vector field on
$\mathcal{T}\mathbb{H}_c^d$ for non-local structural alignment
across topological gaps. The branches are fused via learned gating
and trained with an $N$-pair ranking loss aligned with
retrieval-style evaluation. Across five benchmarks spanning sparse,
dense, featureless, and attributed regimes, \textbf{Concorde}
establishes state-of-the-art early retrieval precision, nearly
doubling $P@10\%$ over leading baselines in sparse settings.
\end{abstract}

\section{Introduction}
Link prediction aims to rank candidate node pairs so that
missing but true edges appear above non-edges~\cite{revisiting}.
Despite strong progress, no single family of methods handles all
structural regimes reliably. Neural approaches, such as GNN-based predictors, subgraph reasoning, and neighborhood-aggregation models, learn flexible link rules from node attributes but degrade sharply under sparsity~\cite{seal,Teru2020,neognn,buddy,ncn}. While classical local overlap heuristics degrade in these regimes due to their strict reliance on dense 2-hop neighborhoods~\cite{newman,adamic,zhou}, global diffusion-style measures remain competitive, encoding multi-hop connectivity without relying on local neighborhood density~\cite{katz,jeh,huang,gelato}.

Geometric deep learning offers a principled route to unifying these
two families. Different graph motifs prefer different curvatures:
tree-like hierarchies are faithfully represented in negatively curved
hyperbolic spaces, while clique-dense subgraphs are better
captured under positive curvature~\cite{poincare,chami,bachmann,gu,
skopek,bai}. Discrete graph curvature (e.g., Ollivier--Ricci)
provides a complementary \emph{intrinsic} descriptor, distinguishing
community edges from bridge-like edges between
regions~\cite{olivier,ye,ni,cgcn}. However, two issues limit the practical impact of geometry
for link prediction. \textbf{(i)} Most geometric methods use a
single manifold and decode distances directly, leaving it unclear
how to couple learned geometry with diffusion operators without
destabilizing the graph's global spectrum~\citep{spectro,prgcn}.
\textbf{(ii)} Most curvature-based methods compute curvature only
on observed edges and do not learn a continuous, feature-driven
similarity that generalizes to unobserved pairs~\citep{curvo,cgcn}.

\paragraph{Contributions:}
\begin{itemize}
    \item \textbf{Spectral decoupling hypothesis.} We prove that a single learned
    adjacency cannot be simultaneously spectrum-preserving (a spectral filter of
    $A$) and $\mathrm{Aut}(G)$-symmetry-breaking (Proposition~\ref{prop:tradeoff}),
    and confirm it directly with a coupled single-matrix control that loses
    $1.8$--$2.5$ AP exactly where the theory predicts
    (Appendix~\ref{app:single_matrix}).
    \item \textbf{A tri-branch energy model.} Concorde decouples link likelihood
    into a parameter-free spectral anchor, a mixed-curvature residual with a
    feature-aware Ollivier--Ricci guardrail, and a non-local hyperbolic steering
    field, fused by learned gates and trained with an $N$-pair ranking loss.
    \item \textbf{State-of-the-art early retrieval.} Across five benchmarks spanning
    sparse/dense and featureless/attributed regimes, Concorde sets a new state of
    the art in early precision, nearly doubling $P@10\%$ over the strongest
    baselines on sparse graphs, with ablations isolating each branch.
\end{itemize}

\definecolor{cA}{HTML}{3B6FB6}   
\definecolor{cR}{HTML}{D9822B}   
\definecolor{cF}{HTML}{2F9468}   
\definecolor{cU}{HTML}{6E55B0}   
\definecolor{cL}{HTML}{C0504D}   
\definecolor{cG}{HTML}{555B66}   

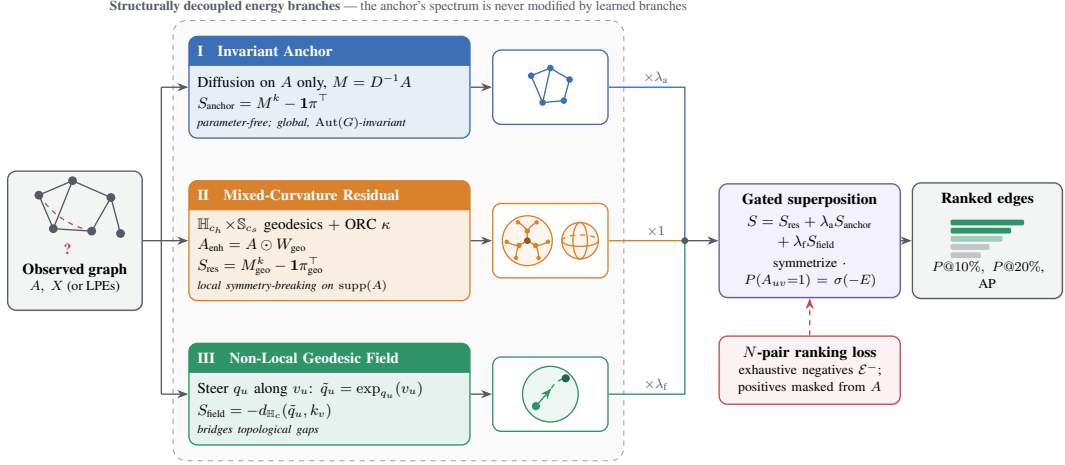
\begin{figure}[t]
\centering
\resizebox{\linewidth}{!}{%
\begin{tikzpicture}[
  font=\small, >=Stealth,
  br/.style={rectangle split, rectangle split parts=2, rounded corners=5pt,
             thick, text width=5.5cm, align=left, inner sep=5pt},
  badge/.style={rounded corners=5pt, thick, fill=white, minimum width=2.4cm,
                minimum height=1.55cm},
  blk/.style={rounded corners=5pt, thick, align=center, inner sep=6pt,
              drop shadow={shadow xshift=0.6pt, shadow yshift=-0.6pt, opacity=0.25}},
  arr/.style={->, thick, cG},
  lab/.style={font=\footnotesize, text=cG, inner sep=2pt}
]

\begin{scope}[shift={(-0.4,0)}]
\draw[rounded corners=5pt, thick, draw=cG, fill=cG!7,
      drop shadow={shadow xshift=0.6pt, shadow yshift=-0.6pt, opacity=0.25}]
     (-1.4,-1.45) rectangle (1.4,1.45);
\begin{scope}[shift={(0,0.5)}]
  \coordinate (n1) at (-0.75,0.45); \coordinate (n2) at (-0.05,0.75);
  \coordinate (n3) at (0.75,0.40);  \coordinate (n4) at (-0.45,-0.25);
  \coordinate (n5) at (0.35,-0.30); \coordinate (n6) at (0.95,-0.35);
  \draw[cG, thick] (n1)--(n2)--(n3)--(n5)--(n4)--(n1) (n2)--(n4) (n3)--(n6);
  \draw[cL, thick, dashed] (n1) to[bend right=20] (n5);
  \foreach \i in {1,...,6}{\fill[cG] (n\i) circle (2.6pt);}
  \node[font=\small\bfseries, text=cL] at (-0.15,-0.66) {?};
\end{scope}
\node[align=center] at (0,-0.78) {\textbf{Observed graph}\\[-1pt]{\footnotesize $A,\;X$ (or LPEs)}};

\end{scope}

\node[br, draw=cA, rectangle split part fill={cA,cA!12}] (b1) at (4.9,3.2)
  {\color{white}\textbf{I \; Invariant Anchor}
   \nodepart{two}
   Diffusion on $A$ only, $M=D^{-1}A$\\[2pt]
   $S_{\text{anchor}}=M^{k}-\mathbf{1}\pi^{\top}$\\[1pt]
   {\scriptsize\itshape parameter-free; global, $\mathrm{Aut}(G)$-invariant}};

\node[br, draw=cR, rectangle split part fill={cR,cR!14}] (b2) at (4.9,0)
  {\color{white}\textbf{II \; Mixed-Curvature Residual}
   \nodepart{two}
   $\mathbb{H}_{c_h}\!\times\!\mathbb{S}_{c_s}$ geodesics $+$ ORC $\kappa$\\[2pt]
   $A_{\text{enh}}=A\odot W_{\text{geo}}$\\[1pt]
   $S_{\text{res}}=M_{\text{geo}}^{k}-\mathbf{1}\pi_{\text{geo}}^{\top}$\\[1pt]
   {\scriptsize\itshape local symmetry-breaking on $\mathrm{supp}(A)$}};

\node[br, draw=cF, rectangle split part fill={cF,cF!12}] (b3) at (4.9,-3.2)
  {\color{white}\textbf{III \; Non-Local Geodesic Field}
   \nodepart{two}
   Steer $q_u$ along $v_u$: $\tilde q_u=\exp_{q_u}(v_u)$\\[2pt]
   $S_{\text{field}}=-d_{\mathbb{H}_c}(\tilde q_u,k_v)$\\[1pt]
   {\scriptsize\itshape bridges topological gaps}};

\node[badge, draw=cA] (i1) at (9.5,3.2) {};
\begin{scope}[shift={(9.5,3.2)}, scale=0.8]
  \coordinate (a1) at (-0.5,0.3); \coordinate (a2) at (0.0,0.5);
  \coordinate (a3) at (0.5,0.2);  \coordinate (a4) at (-0.2,-0.4);
  \coordinate (a5) at (0.35,-0.35);
  \draw[cA, thick] (a1)--(a2)--(a3)--(a5)--(a4)--(a1) (a2)--(a4);
  \foreach \i in {1,...,5}{\fill[cA] (a\i) circle (2.4pt);}
\end{scope}
\node[badge, draw=cR] (i2) at (9.5,0) {};
\begin{scope}[shift={(9.02,0)}]
  \draw[cR, thick, fill=cR!8] (0,0) circle (0.52);
  \foreach \a in {90,210,330}{
    \draw[cR, thick] (0,0) -- (\a:0.2);
    \foreach \o/\b in {-27/14, 27/-14}{
      \draw[cR, thick] (\a:0.2) to[bend left=\b] ({\a+\o}:0.43);
      \fill[cR] ({\a+\o}:0.43) circle (1.3pt);
    }
    \fill[cR] (\a:0.2) circle (1.6pt);
  }
  \fill[cR!70!black] (0,0) circle (2pt);
\end{scope}
\begin{scope}[shift={(10.12,0)}]
  \draw[cR, thick, fill=cR!8] (0,0) circle (0.38);
  \draw[cR, thick] (-0.38,0) arc (180:360:0.38 and 0.12);
  \draw[cR, thick, dashed] (-0.38,0) arc (180:0:0.38 and 0.12);
  \draw[cR, thick] (0,0) ellipse (0.15 and 0.38);
\end{scope}
\node[badge, draw=cF] (i3) at (9.5,-3.2) {};
\begin{scope}[shift={(9.5,-3.2)}]
  \draw[cF, thick, fill=cF!8] (0,0) circle (0.58);
  \fill[cF] (-0.3,-0.25) circle (2.4pt);
  \draw[cF, thick, ->] (-0.3,-0.25) -- (0.02,0.12);
  \draw[cF, thick, dashed] (0.02,0.12) to[bend left=25] (0.32,0.33);
  \fill[cF!70!black] (0.32,0.33) circle (2.4pt);
\end{scope}

\node[blk, draw=cU, fill=cU!8, text width=3.35cm] (fu) at (14.9,0)
  {\textbf{Gated superposition}\\[3pt]
   $\begin{aligned}S&=S_{\text{res}}+\lambda_{\text{a}}S_{\text{anchor}}\\[-1pt]&\phantom{=}\;+\lambda_{\text{f}}S_{\text{field}}\end{aligned}$\\[3pt]
   {\footnotesize symmetrize $\cdot$ $P(A_{uv}{=}1)=\sigma(-E)$}};

\node[blk, draw=cG, fill=cG!7, text width=2.7cm] (out) at (18.6,0)
  {\textbf{Ranked edges}\\[3pt]
   \begin{tikzpicture}[baseline=0pt]
     \foreach \w/\c [count=\i from 0] in {1.7/cF, 1.4/cF, 1.2/cF!45, 0.9/cG!35, 0.7/cG!35}{
       \fill[\c, rounded corners=1pt] (0,-0.17*\i) rectangle (\w*0.9,-0.17*\i+0.12);}
   \end{tikzpicture}\\[1pt]
   {\footnotesize $P@10\%,\;P@20\%$, AP}};

\node[blk, draw=cL, fill=cL!7, text width=3.35cm] (loss) at (14.9,-2.7)
  {\textbf{$N$-pair ranking loss}\\[1pt]
   {\footnotesize exhaustive negatives $\mathcal{E}^-$;\\ positives masked from $A$}};

\draw[cG, thick] (1.0,0) -- (1.4,0);
\draw[cG, thick] (1.4,3.2) -- (1.4,-3.2);
\draw[arr] (1.4,3.2)  -- (b1.west);
\draw[arr] (1.4,0)    -- (b2.west);
\draw[arr] (1.4,-3.2) -- (b3.west);

\draw[arr] (b1.east) -- (i1.west);
\draw[arr] (b2.east) -- (i2.west);
\draw[arr] (b3.east) -- (i3.west);

\draw[cA, thick] (i1.east) -- node[lab,above,pos=.62]{$\times\lambda_{\text{a}}$} (12.3,3.2) -- (12.3,0);
\draw[cR, thick] (i2.east) -- node[lab,above,pos=.62]{$\times 1$} (12.3,0);
\draw[cF, thick] (i3.east) -- node[lab,above,pos=.62]{$\times\lambda_{\text{f}}$} (12.3,-3.2) -- (12.3,0);
\draw[arr] (12.3,0) -- (fu.west);
\fill[cG] (12.3,0) circle (2pt);

\draw[arr] (fu.east) -- (out.west);
\draw[arr, dashed, cL] (loss.north) -- (fu.south);

\begin{scope}[on background layer]
  \node[draw=cG!60, dashed, rounded corners=8pt, fill=cG!3, inner sep=9pt,
        fit=(b1)(b2)(b3)(i1)(i2)(i3)] (grp) {};
\end{scope}
\node[font=\footnotesize, text=cG, anchor=south] at ([yshift=1pt]grp.north)
  {\textbf{Structurally decoupled energy branches} --- the anchor's spectrum is never modified by learned branches};

\end{tikzpicture}%
}
\caption{\textbf{Overview of Concorde.} Link likelihood is decomposed into three
structurally decoupled energy branches: a parameter-free spectral anchor on the
unperturbed adjacency (global, $\mathrm{Aut}(G)$-invariant), a mixed-curvature
residual on $\mathbb{H}\times\mathbb{S}$ with a feature-aware Ollivier--Ricci
guardrail (local symmetry-breaking), and a non-local geodesic steering field on
$\mathcal{T}\mathbb{H}_c^d$ (bridging topological gaps). Learned scalar gates fuse
the branches, and the model is trained with an $N$-pair ranking loss aligned with
early-retrieval evaluation.}
\label{fig:overview}
\end{figure}

\subsection{Related Work}
Classical link predictors based on neighborhood overlap, such as Common
Neighbors, Adamic--Adar, and Resource Allocation, remain strong
baselines due to their direct encoding of triadic closure~\cite{newman,
adamic,zhou}. Global diffusion measures such as Katz~\cite{katz} and
SimRank~\cite{jeh} exploit multi-hop connectivity via path-based
similarity. Autocovariance (AC) captures random-walk co-visitation
structure and can outperform learned GNN baselines on sparse
graphs~\cite{huang}. Gelato refines a diffusion heuristic with a
light attribute-driven learner and ranking objective, demonstrating
that heuristic reweighting is extremely effective under heavy
imbalance~\cite{gelato}.

A second line of work performs edge scoring via explicit local-structure
reasoning. SEAL popularized enclosing-subgraph classification; subsequent methods improved scalability via neighborhood sketching~\citep{buddy} or modeled explicit common-neighbor structure~\citep{neognn}. NCN/NCNC incorporate common-neighbor
structure with false-negative completion~\citep{ncn}, and OCN
further refines these signals via orthogonalization to address
redundancy and over-smoothing~\citep{ocn}.

Geometric representation learning embeds nodes in non-Euclidean
spaces to match global graph structure. Hyperbolic embeddings
compactly model hierarchies via exponential volume
growth~\cite{poincare}, and hyperbolic GNNs extend message passing
to these manifolds~\cite{chami}. Unified constant-curvature
frameworks~\cite{bachmann} and mixed-curvature product
spaces~\cite{gu,skopek} capture multiple structural motifs
simultaneously, though most rely on single decoders and do not
address spectral brittleness when coupling learned geometry to
diffusion operators. More recent work combines mixed-curvature geometry with
global attention in product-space transformers, or learns per-component
curvature importance rather than fixing it  \textit{a priori} ~\cite{gu,skopek,dwi}.
Concorde differs from existing literature in two ways: its non-local branch replaces Euclidean
self-attention with a \emph{geodesic} steering field native to the hyperbolic
manifold (Section~\ref{sec:branch_3}), and its learned gates weight
\emph{structurally decoupled} energy branches rather than reweighting components
inside one representation, which is precisely what preserves the anchor spectrum
(Proposition~\ref{prop:anchor_stability}).

Furthermore, discrete Ricci curvature characterizes local graph
``shape,'' distinguishing community edges from topological
bridges~\cite{olivier,ni}. Curvature Graph Networks adapt edge
weights via Ollivier--Ricci curvature during message
passing~\cite{ye}, discrete curvature has been used to characterize
over-squashing and bottlenecks in message passing~\cite{topping2021},
and curvature has been used as an auxiliary
signal in applied link prediction~\cite{cgcn}. These methods
provide valuable intrinsic descriptors but do not learn
task-adaptive similarity from node features or integrate with
diffusion-based pipelines.

\section{Methodology}

We introduce \textbf{Concorde}, an energy-based framework that decouples link likelihood into three structurally restricted branches, each targeting a distinct functional regime.

Let $G=(V,E)$ be an undirected graph with adjacency $A \in \mathbb{R}^{N \times N}$ and features $X \in \mathbb{R}^{N \times d}$; link prediction scores each pair $(u,v)$ by a structural similarity $S(u,v)$~\citep{revisiting}. Rather than deriving $S$ from a distance decoder in a single latent space \cite{hoff02, sarkar, poincare}, we treat the negative similarity as a \textit{geometric energy} $E(u,v)$ and model link formation as a Boltzmann distribution with a logistic link:
\begin{equation}
    P(A_{uv} = 1) = \sigma(-E(u,v)).
\end{equation}

Low energy thus corresponds to a high-probability missing link. Standard learned methods minimize this energy within a single representation space or enhanced adjacency $A_{\text{learn}}$~\cite{chen,franchesci}. We argue this forces an unavoidable spectral trade-off between two opposing group-theoretic objectives: \textbf{global invariance} (spectral filters on $A$ are permutation-invariant w.r.t.\ $\mathrm{Aut}(G)$) and \textbf{local equivariance} (detecting missing links requires distinguishing structurally symmetric candidates via localized geometric deformations).

We formalize this as \textbf{a spectral decoupling hypothesis}. If the single
learned operator $A_{\mathrm{learn}}$ acts as a \emph{spectral filter}
$\sum_k c_k A^k$ of $A$, as every diffusion operator (including the anchor's
$M^k$) does, it is automatically $\mathrm{Aut}(G)$-equivariant, the standard
condition satisfied by any message-passing GNN~\cite{xu2019,maron}, and assigns
identical scores within each $\mathrm{Aut}(G)$-orbit. The symmetry-breaking that
link prediction requires is thus impossible without leaving this spectral class
and distorting the spectrum (Proposition~\ref{prop:tradeoff},
Appendix~\ref{app:spectral_theory}). The
trade-off is thus structurally generic, not a pathological case. Because
informative node features can partially break $\mathrm{Aut}(G)$-symmetry,
Proposition~\ref{prop:tradeoff} characterizes the \textit{structural}
component of the learned signal: the impossibility is most binding where
features are weak or absent relative to topology. Concorde abandons the
single-matrix paradigm accordingly, a choice supported by the theory
(Appendix~\ref{app:spectral_theory}), by the performance collapse under
ablation (Section~\ref{sec:macro_ablation}), and by a direct coupled
single-matrix experiment reproducing exactly the regime
Proposition~\ref{prop:tradeoff} argues against
(Appendix~\ref{app:single_matrix}). We define the total geometric
energy $E(u,v)$ as a decoupled superposition of three distinct potentials,
each restricted to a specific functional space. By setting the final
pairwise score $S(u,v)$ to the negative of this joint energy, Concorde
computes link probabilities through a Tri-Branch architecture:

\begin{equation}
    S(u,v) = -E(u,v) = S_{\text{anchor}}(u,v) + S_{\text{res}}(u,v) + S_{\text{field}}(u,v).
\end{equation}

The three branches are: (I) a parameter-free spectral anchor that captures global topological invariance; (II) a mixed-curvature residual on a learnable $\mathbb{H}_{c_h}^{d_h} \times \mathbb{S}_{c_s}^{d_s}$ product manifold for local geometric symmetry-breaking; and (III) a continuous vector field on $\mathcal{T}\mathbb{H}_c^d$ for non-local structural alignment across topological gaps. We discuss the three branches next.

\subsection{Branch I: Global Topological Energy (The Invariant Anchor)}
\label{sec:branch_1}

The first branch of the Concorde architecture establishes the macroscopic stability of the model. In multiscale geometric frameworks, stability is typically maintained by proving that a hierarchy of coarsened graphs remains quasi-isometric to the original topology \cite{loukas,zou}. Concorde, however, achieves this stability natively in a single scale by strictly isolating the macroscopic topology into an invariant anchor. This branch minimizes the \textit{global topological energy}, measuring structural affinity between nodes from the observed network infrastructure alone. Graph automorphisms,  permutations preserving adjacency, act as linear transformations on the disjoint eigenspaces of $A$; consequently, any parameter-free spectral function of $A$ is invariant to these symmetries by construction.

Let $D$ be the diagonal degree matrix of the training graph where $D_{uu} = \sum_{v} A_{uv}$, and define the row-stochastic transition matrix $M = D^{-1}A$. To evaluate the structural energy between nodes $u$ and $v$ over a receptive field of $k$ hops, we model the multi-step diffusion process. Let $\pi \propto D\mathbf{1}$ represent the stationary distribution column vector of $M$, and let $\mathbf{1}$ be the all-ones column vector. The invariant anchor score matrix is formally defined as a
degree-normalized diffusion residual:
\begin{equation}
    S_{\text{anchor}} = M^k - \mathbf{1}\pi^T.
\end{equation}

Each entry $[S_{\text{anchor}}]_{uv} = [M^k]_{uv} - \pi_v$ measures the
\emph{excess} $k$-step reachability of $v$ from $u$ above the stationary
baseline $\pi_v$. This is related to, but distinct from, the standard
Autocovariance (AC) heuristic of~\citet{huang}, which uses the
degree-weighted form $\Pi M^k - \pi\pi^T$ where $\Pi = \text{diag}(\pi)$.
Our row-stochastic formulation omits the $\Pi$ pre-weighting
deliberately: the $\Pi$-weighted form amplifies hub nodes in
proportion to their stationary probability, introducing precisely
the popularity bias $S_{\text{anchor}}$ is designed to suppress. The next two branches supply the fine-grained symmetry-breaking signal that this parameter-free anchor cannot provide by construction.

\textbf{Well-posedness on real graphs:} A unique stationary $\pi$ requires the
walk to be aperiodic; we ensure this with self-loops (a standard lazy-walk
device, enabled on the datasets most prone to periodicity), which also give
dangling low-degree nodes well-defined stationary mass. On disconnected graphs,
$M = D^{-1}A$ is block-diagonal, so Branch~I operates independently within each
connected component ($\pi \propto D\mathbf{1}$ is well-defined per block) and
never scores cross-component pairs.

\subsection{Branch II: Local Geometric Energy (Mixed-Curvature Residual)}
\label{sec:branch_2}

While the invariant anchor stabilizes the macroscopic ranking, it is blind to
localized symmetry-breaking: nodes in the same structural orbit receive
identical scores, yet link prediction hinges on detecting localized
perturbations such as a closing triangle or a missing cross-community bridge.
Concorde captures these high-frequency corrections without destabilizing the
global spectrum through a \textit{local geometric energy} branch on a
mixed-curvature manifold.

Real-world graphs are geometrically heterogeneous: tree-like regions embed best
in negative curvature, while cyclic or clique-dense subgraphs favor positive
curvature~\cite{bronstein, sala}, and no single Euclidean or purely hyperbolic
space accommodates both without distortion~\cite{gu}. We therefore map node
features into a product manifold $\mathcal{M} = \mathbb{H}_{c_h}^{d_h} \times
\mathbb{S}_{c_s}^{d_s}$ with learnable curvatures $c_h < 0$ and $c_s > 0$, and
compute coordinate-free geodesic distances $d_{\mathbb{H}}(u,v)$ and
$d_{\mathbb{S}}(u,v)$ on each component. Learning $c_h, c_s$ end-to-end lets the
model scale each component to the dataset's intrinsic geometry rather than
fixing it \textit{a priori}.

Geodesic proximity is necessary but not sufficient: distances alone cannot
separate a coherent edge from a spurious bridge between incompatible
neighborhoods. We therefore add a feature-aware Ollivier--Ricci curvature (ORC)
proxy $\kappa(u,v)$ measuring neighborhood alignment, with positive curvature
indicating community cohesion and negative curvature a topological bottleneck or
bridge.

Formally, we define lazy random-walk measures $\mu_u$ and $\mu_v$
over the one-hop neighborhoods of $u$ and $v$ respectively. Each
measure assigns a retention mass $\alpha \in (0,1)$ to the node
itself and distributes the remaining $1-\alpha$ uniformly over its
immediate neighbors $\mathcal{N}(\cdot)$:
\begin{equation}
    \mu_u(x) = \alpha \cdot \mathbf{1}[x = u] \;+\;
    \frac{1-\alpha}{|\mathcal{N}(u)|}\sum_{w \in \mathcal{N}(u)}
    \mathbf{1}[x = w].
\end{equation}
The curvature is then defined via optimal transport between these
localized distributions:
\begin{equation}
    \kappa(u,v) = 1 - \mathcal{W}_1(\mu_u, \mu_v),
\end{equation}
where $\mathcal{W}_1$ denotes the Wasserstein-1 distance. To ensure scalability to large graphs, this transport cost is approximated via the highly parallelizable entropy-regularized Sinkhorn algorithm \cite{sinkhorn} and precomputed for candidate batches. For $\kappa(u,v)$ to reflect both structural and semantic compatibility, the ground metric must account for more than hop distance.
We define a joint ground metric $C(a,b)$ as a convex combination of topological and semantic dissimilarity:
\begin{equation}
C(a,b) = \frac{1}{2}d_{\text{topo}}(a,b) + \frac{1}{2}(1 - \text{cos\_sim}(X_a, X_b)), \quad \text{where } d_{\text{topo}}(a,b) =
\begin{cases}
0 & \text{if } a = b \\
1 & \text{if } (a,b) \in E \\
2 & \text{otherwise}
\end{cases}.
\end{equation}
This coupling ensures that the transport cost, and by extension
$\kappa(u,v)$, penalizes pairs whose neighborhoods are
geometrically proximate yet semantically incompatible, providing
an intrinsic check against structurally implausible link predictions. We fuse the extrinsic manifold distances and the intrinsic ORC guardrail via a deep neural function $\mathcal{F}_\theta$ to compute a local structural modulator:

Let $\bar{d}_{\mathbb{H}}(u)$ and $\bar{d}_{\mathbb{S}}(u)$ be the mean geodesic distances from $u$ to its observed neighbors, encoding local manifold density. The structural modulator $W_{\text{geo}} \in \mathbb{R}^{N \times N}$ has entries:
\begin{equation}
[W_{\text{geo}}]_{uv} = \sigma\Big(\mathcal{F}_\theta\big(
\tfrac{X_u + X_v}{2} \oplus |X_u - X_v|,\;
d_{\mathbb{H}}(u,v),\; d_{\mathbb{S}}(u,v),\; \kappa(u,v),\;
\bar{d}_{\mathbb{H}}(u),\; \bar{d}_{\mathbb{H}}(v),\;
\bar{d}_{\mathbb{S}}(u),\; \bar{d}_{\mathbb{S}}(v)
\big)\Big).
\end{equation}

Here $\tfrac{X_u+X_v}{2}$ and $|X_u - X_v|$ decompose the pair into symmetric and difference feature components. In strictly featureless regimes (e.g., OGBL-DDI), we substitute the absent node attributes $X$ with Laplacian Positional Encodings (LPEs) to provide initial structural coordinates. Because the graph is undirected, the resulting pairwise relationships must be perfectly symmetric. We therefore explicitly symmetrize the structural modulator matrix:
\begin{equation}
    W_{\text{geo}} \leftarrow \frac{1}{2}\left(W_{\text{geo}} + W_{\text{geo}}^\top\right).
\end{equation}

This symmetrized, bounded affinity matrix acts as a localized, equivariant deformation of the graph space. We apply this modulator to reweight the observed topology via the Hadamard product, constructing an enhanced adjacency matrix $A_{\text{enh}} = A \odot W_{\text{geo}}$.

Defining $D_{\text{enh}}$ as the corresponding diagonal degree matrix where $[D_{\text{enh}}]_{uu} = \sum_v [A_{\text{enh}}]_{uv}$, we construct the enhanced transition matrix $M_{\text{geo}} = D_{\text{enh}}^{-1}A_{\text{enh}}$. This dynamically injects high-frequency structural cues into the transition probabilities while strictly preserving the sparsity of the original graph. The residual energy score is then extracted via diffusion over this augmented space (with $\pi_{\text{geo}}$ as its stationary distribution):
\begin{equation}
    S_{\text{res}} = M_{\text{geo}}^k - \mathbf{1}\pi_{\text{geo}}^T.
\end{equation}

The Hadamard form $A \odot W_{\text{geo}}$ amplifies structurally consistent edges and dampens spurious ones while keeping $O(|\mathcal{E}|)$ cost. Adding the unperturbed anchor ($S_{\text{anchor}} + S_{\text{res}}$) acts as a stability regularizer: the global ranking cannot fall below the baseline topological performance. This enforces a separation of concerns, the anchor supplies low-frequency global topology, the residual applies feature-adaptive local reweighting confined to the sparsity support of $A$, and Branch~III resolves unobserved gaps.

\subsection{Branch III: Non-Local Field Energy (Contextual Geodesic Steering)}
\label{sec:branch_3}

Branches I and II capture structure within connected subgraphs, but diffusion breaks down under extreme sparsity: two nodes may share identical structural orbits yet be entirely disconnected. Standard attention bridges this via pairwise affinities, but its Euclidean inner products ($QK^T$) are geometrically inconsistent with the curved manifold of Branch~II~\cite{veli, gulcehre, dwi}. Concorde instead learns a \textit{continuous feature-to-tangent map} $\mathcal{V}: \mathbb{R}^d \rightarrow \mathcal{T}\mathbb{H}_c^d$ that assigns each query node a tangent vector pointing toward its most probable missing link, steering natively on the manifold without neighborhood traversal or costly parallel transport.

First, we project the node features into the tangent space of the origin
and lift them to the Lorentz model to form query and key representations
$q_u, k_v \in \mathbb{H}_c^d$. In featureless settings, the same
Laplacian Positional Encodings (LPEs) used in Branch II serve as the
input to the feature-to-tangent map $\mathcal{F}_{\text{steer}}$;
crucially, since LPEs are precomputed on the full training graph and
fixed throughout training, per-batch edge masking does not require
their recomputation, keeping the overhead of Branch III identical
across attributed and featureless regimes. For each node $u$, the network generates a specific tangent steering vector:
\begin{equation}
    v_u = \text{proj}_{q_u}\big(\mathcal{F}_{\text{steer}}(X_u)\big) \in \mathcal{T}_{q_u}\mathbb{H}_c^d,
\end{equation}
where the projection operator ensures the steering vector lies strictly in the tangent space of $q_u$ (i.e., the Minkowski inner product $\langle q_u, v_u \rangle_{\mathcal{L}} = 0$). To evaluate the structural alignment between the query's intention and the candidate key, we utilize the Lorentz Exponential Map. This operator transports the query node along the exact geodesic trajectory defined by the semantic steering vector (using $|c_h|$ to denote the magnitude of the negative curvature):
\begin{equation}
    \tilde{q}_u = \exp_{q_u}(v_u) = \cosh(\sqrt{|c_h|}||v_u||_{\mathcal{L}})q_u + \frac{\sinh(\sqrt{|c_h|}||v_u||_{\mathcal{L}})}{\sqrt{|c_h|}||v_u||_{\mathcal{L}}}v_u.
\end{equation}

Finally, the \textit{Non-Local Field Energy} is computed as the geodesic distance between the semantically steered query and the key across the manifold. We convert this energy into an attention-based similarity score:
\begin{equation}
S_{\text{field}}(u,v) = -d_{\mathbb{H}_c}(\tilde{q}_u, k_v).
\end{equation}

A lower geodesic distance between the steered query $\tilde{q}_u$
and the key $k_v$ corresponds to a higher structural affinity, so
the negative distance serves directly as a pairwise similarity score. To capture multi-faceted structural signals, we implement this mechanism using a multi-head attention architecture, where each head learns a distinct steering field and an independent curvature constant (separate from the base manifold in Branch II), which are then fused via a learnable weighted sum.

\textbf{Directedness and parameterization:} The query and key representations are
produced by an \emph{independent} feature-to-tangent lift, not shared with the
Branch~II encoder, leaving the steering field free to specialize for non-local
alignment. The score $S_{\text{field}}(u,v) = -d_{\mathbb{H}_c}(\tilde{q}_u, k_v)$
is deliberately directed, it steers a query toward its most probable partner and
is not symmetric in $(u,v)$. Because the target graph is undirected, symmetry is
restored at fusion via $S(u,v)\leftarrow\tfrac{1}{2}(S(u,v)+S(v,u))$
(Section~\ref{sec:optimization}): the directed field enriches the hypothesis
space during scoring while the final energy stays permutation-invariant.

\subsection{Optimization \& Multi-hop Superposition}
\label{sec:optimization}
To compute the final link likelihood, Concorde integrates
the three decoupled energy branches into a unified scoring function.
Rather than fixing the relative contribution of each branch, we learn
a dynamic superposition via end-to-end gating:
\begin{equation}
    S(u,v) = S_{\text{res}}(u,v) + \lambda_{\text{anchor}} \,
    S_{\text{anchor}}(u,v) + \lambda_{\text{field}} \,
    S_{\text{field}}(u,v),
\end{equation}
where $\lambda_{\text{anchor}}, \lambda_{\text{field}} \in (0,1)$ are
sigmoid-parameterized scalar gates optimized end-to-end and $S_{\text{res}}$ is
fixed at coefficient $1$ as the baseline energy scale. Since the graph is
undirected, we symmetrize the final score,
$S(u,v) \leftarrow \tfrac{1}{2}(S(u,v) + S(v,u))$, resolving the asymmetry in
Branch~I's diffusion residual and Branch~III's directed steering. The gates let
the model adapt to dataset sparsity and homophily: dense graphs upweight
$S_{\text{anchor}}$, sparse or heterophilic graphs upweight $S_{\text{field}}$.
Empirically the learned gates recover this dependence without supervision,
$\lambda_{\text{field}}$ rising monotonically as the graph sparsifies ($0.10$ on
dense OGBL-DDI to $0.63$ on sparse Power), and outperform both equal and
oracle-tuned fixed gates (Appendix~\ref{app:extra_ablations}).

In sparse completion, non-edges grow quadratically with $|V|$, so calibrated
binary classification is ill-posed and cross-entropy conflates ranking with
calibration. Since our target is retrieval, a ranked shortlist of candidate
edges, only the \textit{relative ordering} of scores matters. We therefore
supervise ordering directly with an $N$-pair ranking loss over true positives
$\mathcal{E}^+$ and the exhaustive complement of unobserved negatives
$\mathcal{E}^-$ (no subsampling):
\begin{equation}
    \mathcal{L} = \frac{1}{|\mathcal{E}^+|} \sum_{(u,v) \in \mathcal{E}^+}
    \log \Bigg( 1 + \sum_{(u',v') \in \mathcal{E}^-}
    \exp\Big(S(u',v') - S(u,v)\Big) \Bigg).
\end{equation}
The exhaustive negative complement is made memory-feasible via graph
partitioning and chunked evaluation (Appendix~\ref{app:scalability}). Penalizing
any negative that scores above a true positive directly optimizes early-retrieval
precision ($P@10\%$, $P@20\%$). To prevent structural leakage, we mask each
mini-batch's positive edges $\mathcal{E}^+$ from $A$ before constructing the
anchor and residual, so the diffusion branches infer link likelihood only from
surrounding context. This masking is exact for the diffusion operators; on
featureless graphs the fixed LPE inputs are exempt, but their residual leakage
is bounded (Davis--Kahan) and empirically negligible ($\leq 0.44$ AP;
Appendices~\ref{app:eval_protocol},~\ref{app:extra_ablations}).

\section{Experiments}
\label{sec:experiments}

We evaluate Concorde on five benchmark datasets spanning
distinct structural and semantic regimes, following the taxonomy
of~\citet{revisiting}. \textit{Cora}, \textit{CiteSeer}~\cite{cora},
and \textit{PubMed}~\cite{pubmed} are citation networks in which both
global structural proximity and bag-of-words node features are
informative, providing a rich testbed for methods that integrate
structural and semantic signals. \textit{OGBL-DDI}~\cite{ogb} is a
dense, featureless drug-drug interaction network governed exclusively
by local neighborhood overlap. \textit{Power}~\cite{power}
is an electrical grid topology at the opposite extreme: extremely
sparse, featureless, and dominated by global path-based signals
(e.g., Katz index) due to its low mean degree. These five are chosen to span the
structural taxonomy along two axes, density (sparse Power to dense OGBL-DDI) and
feature availability (featureless to attributed), so that each branch is stressed
in isolation; OGBL-DDI's ${\sim}1.3$M edges additionally test scale, and
Concorde's $O(k|E|)$ per-step cost (Appendix~\ref{app:scalability}) carries the
identical protocol to larger graphs.

We compare against a comprehensive set of link prediction baselines. Enclosing subgraph methods include SEAL~\cite{seal}, which learns structural heuristics from localized subgraphs, and BUDDY~\cite{buddy}, a scalable subgraph-sketching alternative. Neighborhood overlap methods, including Neo-GNN~\cite{neognn}, NCN, NCNC~\cite{ncn}, and the previous state-of-the-art OCNP~\cite{ocn}, explicitly incorporate common-neighbor counts and higher-order overlaps on dense, triadic-closure-heavy graphs. We additionally include HGCN~\cite{chami}, a hyperbolic GNN with Fermi-Dirac link decoder, as a single-manifold geometric baseline against which the product manifold design of Branch~II is directly contrasted. Finally, representing the previous state-of-the-art among diffusion-based methods, Gelato~\cite{gelato} frames link prediction as an imbalanced retrieval problem, augmenting topological heuristics with attribute-based graph learning. We report \textbf{Average Precision (AP)} as a measure of global ranking quality, with primary emphasis on \textbf{Early Precision ($P@10\%$, $P@20\%$)}. We emphasize early precision because retrieval under extreme class imbalance is governed by the density of true edges at the very top of a \emph{single global ranking} over the entire non-edge complement, which is what $P@r\%$ measures with a denominator held identical across methods, rather than by the sampled-negative shortlists that Hits@$K$/MRR are typically computed over; our exhaustive full-negative evaluation already imposes the hard-negative pressure that per-source protocols such as HeaRT approximate, and AP is reported throughout for global calibration (Appendix~\ref{app:eval_protocol}). All baselines are run by us under the \emph{identical} unbiased full-negative protocol and hardware; entries marked OOM or exceeding the time budget reflect a method's cost under this harder protocol rather than selective omission, as the per-method runtime and peak-memory tables in Appendix~\ref{app:scalability} make explicit.

\subsection{Main Results: Breaking the Sparsity Bottleneck}
\label{sec:main_results}

\begin{table}[htbp]
\centering
\caption{Link prediction performance comparison across five benchmark datasets. We report Early Precision ($P@10\%$, $P@20\%$) and Average Precision (AP) using an unbiased evaluation protocol. Results show mean $\pm$ standard deviation over 10 runs, except where noted. ($^{*}$) Run only once due to
runtime $>24$ hrs; OOM: Out-of-memory; ($^{**}$) Each run takes $>1000$ hrs \cite{gelato}}
\label{tab:main_results}
\scriptsize
{%
\begin{tabular}{ll ccccc}
\toprule
Metric & Model & Cora & CiteSeer & PubMed & OGBL-DDI & Power \\
\midrule
\multirow{8}{*}{P@10}
 & SEAL     & 6.52 $\pm$ 1.08
 & 5.35 $\pm$ 1.12 & $^{**}$ & 37.45 $\pm$ 1.15 & 11.34 $\pm$ 1.28 \\
 & HGCN &  6.53 $\pm$ 1.11 & 7.25 $\pm$ 1.23 & 11.10 $\pm$ 1.25 & 27.32 $\pm$ 1.93 & 18.28 $\pm$ 1.19 \\
 & Neo-GNN  & 6.05 $\pm$ 0.32 & 6.23 $\pm$ 0.85 & 12.32 $\pm$ 0.74 & 25.80* & 10.74 $\pm$ 0.39 \\
 & BUDDY    & 7.23 $\pm$ 0.45 & 7.93 $\pm$ 0.80 & OOM & 22.59 $\pm$ 2.02 & 12.53 $\pm$ 0.45 \\
 & NCN      & 10.12 $\pm$ 1.08 & 10.36 $\pm$ 0.32 & 3.23 $\pm$ 1.09 & 45.30 $\pm$ 0.15 & 14.77 $\pm$ 0.61 \\
 & NCNC     & 8.78 $\pm$ 0.58 & 10.84 $\pm$ 0.34 & 14.56 $\pm$ 0.89 & 45.22 $\pm$ 0.11 & 13.42 $\pm$ 0.31 \\
 & OCNP     & 9.30 $\pm$ 0.32 & 12.76 $\pm$ 0.27 & 17.38 $\pm$ 0.12 & 46.12 $\pm$ 0.13 & 20.60 $\pm$ 0.60 \\
 & Gelato   & 13.21 $\pm$ 1.33  &  8.70 $\pm$ 1.05 & 22.07 $\pm$ 0.12 & 33.34 $\pm$ 0.11 & 24.85 $\pm$ 4.49 \\
 & Concorde & \textbf{25.28 $\pm$ 1.92} & \textbf{20.00 $\pm$ 1.58} & \textbf{23.69 $\pm$ 0.15} & \textbf{46.45 $\pm$ 0.20} & \textbf{27.58 $\pm$ 2.78} \\
\midrule
\multirow{8}{*}{P@20}
 & SEAL     & 5.57$ \pm$ 0.98 & 5.08 $\pm$ 0.87 & $^{**}$ & 36.23 $\pm$ 1.30 & 10.12 $\pm$ 0.38 \\
  & HGCN &  5.48 $\pm$ 1.20 & 6.88 $\pm$ 0.92 & 10.53 $\pm$ 0.97 & 26.60 $\pm$ 2.13 & 17.72 $\pm$ 0.98 \\
 & Neo-GNN  & 5.40 $\pm$ 0.45 & 6.12 $\pm$ 0.37 & 10.54 $\pm$ 0.67 & 24.28* &  11.11 $\pm$ 0.25 \\
 & BUDDY    & 6.23 $\pm$ 0.71 & 7.15 $\pm$ 0.32 & OOM & 20.34 $\pm$ 1.15 &  12.38 $\pm$ 0.48 \\
 & NCN      & 9.08 $\pm$ 0.46 & 9.71 $\pm$ 0.38 & 2.21 $\pm$ 1.84 & 44.30 $\pm$ 0.08 &  13.77 $\pm$ 0.61 \\
 & NCNC     & 7.34 $\pm$ 0.48 & 9.12 $\pm$ 0.75 & 11.28 $\pm$ 0.82 & 44.85 $\pm$ 0.09 &  14.45 $\pm$ 0.87 \\
 & OCNP     & 8.84 $\pm$ 0.23 & 10.11 $\pm$ 0.43 & 13.21 $\pm$ 0.24 & 46.23 $\pm$ 0.08 &  19.91 $\pm$ 1.15 \\
 & Gelato   & 15.47 $\pm$ 0.52 & 14.75 $\pm$ 0.62 & 16.80 $\pm$ 0.24 & 32.25 $\pm$ 0.07 & 23.48 $\pm$ 2.57 \\
 & Concorde & \textbf{20.00 $\pm$ 0.71} & \textbf{16.48 $\pm$ 0.69} & \textbf{18.92 $\pm$ 0.12} & \textbf{46.48 $\pm$ 0.08} & \textbf{24.24 $\pm$ 1.36} \\
\midrule
\multirow{8}{*}{AP}
 & SEAL     & 2.27 $\pm$ 0.45 & 2.52 $\pm$ 0.22 & $^{**}$ & 35.20 $\pm$ 1.06 & 2.39 $\pm$ 0.12  \\
   & HGCN &  2.53 $\pm$ 0.27 & 2.77 $\pm$ 0.38 & 1.84 $\pm$ 0.21 & 22.28 $\pm$ 2.07 & 3.87 $\pm$ 0.49 \\
 & Neo-GNN  & 2.08 $\pm$ 1.68 & 1.68 $\pm$ 0.09 & 1.31 $\pm$ 0.24 & 23.39* & 2.82 $\pm$ 0.18 \\
 & BUDDY    & 1.32 $\pm$ 0.40 & 1.90 $\pm$ 0.15 & OOM & 18.89 $\pm$ 1.09 & 3.14 $\pm$ 0.40 \\
 & NCN      & 1.98 $\pm$ 0.34 & 2.85 $\pm$ 0.14 & 0.13 $\pm$ 0.05 & 44.75 $\pm$ 0.07 & 3.38 $\pm$ 0.27 \\
 & NCNC     & 2.54 $\pm$ 0.16 & 3.10 $\pm$ 0.12 & 1.90 $\pm$ 0.02 & 44.89 $\pm$ 0.09 & 3.50 $\pm$ 0.24 \\
 & OCNP     & 2.88 $\pm$ 0.20 & 3.30 $\pm$ 0.15 & 2.21 $\pm$ 0.03 & \textbf{46.89 $\pm$ 0.07} & 3.85 $\pm$ 0.37  \\
 & Gelato   & 3.90 $\pm$ 0.03 & 4.55 $\pm$ 0.02 & \textbf{2.88 $\pm$ 0.00} & 29.42 $\pm$ 0.00 & 4.33 $\pm$ 0.32 \\
 & Concorde & \textbf{4.50 $\pm$ 0.03} & \textbf{5.05 $\pm$ 0.18} & 2.83 $\pm$ 0.02 & 42.32 $\pm$ 0.14 & \textbf{6.39 $\pm$ 0.28} \\
\bottomrule
\end{tabular}%
}
\end{table}

Table~\ref{tab:main_results} summarizes results across all five datasets,
where Concorde either establishes a new state of the art or matches models
specialized for a specific topological regime. Under extreme sparsity, true
missing links frequently span multi-hop gaps with no common neighbors, and
neighborhood-overlap methods degrade sharply (NCN $14.77\%$, NCNC $13.42\%$
$P@10\%$). Concorde reaches $27.58\%$ $P@10\%$ on Power, nearly doubling
NCN/NCNC; the gain is attributable to Branch~III (Section~\ref{sec:branch_3}),
whose learned feature-to-tangent map aligns node pairs across topological gaps
without any path-based connectivity signal.

OGBL-DDI is the setting most favorable to explicit common-neighbor models: a
dense, featureless network governed by local triadic closure. Concorde
nonetheless surpasses the strongest overlap baseline OCNP at the top of the
retrieval list ($46.45$/$46.48\%$ $P@10\%$/$P@20\%$ vs.\ $46.12$/$46.23\%$),
though OCNP retains an AP advantage ($46.89\%$ vs.\ $42.32\%$). This gap is
structurally expected: the $N$-pair loss optimizes score \emph{ordering} among
top candidates rather than global calibration, so it reflects a deliberate bias
toward early retrieval, not a failure of discriminability. Notably, Concorde
matches these specialized overlap models at top-$k$ precision using only the
spherical manifold component and positional encodings, without explicit
neighborhood enumeration or feature supervision.

With rich node features, Concorde achieves substantial early-retrieval gains:
$P@10\%$ of $25.28\%$ on Cora (vs.\ $13.21\%$ for Gelato, the strongest
similarity-based method) and $20.00\%$ on CiteSeer. The invariant anchor
stabilizes global citation structure while the feature-to-tangent map resolves
unobserved intra-community links. On PubMed, Concorde leads in early precision
($23.69\%$ vs.\ $22.07\%$) but Gelato retains a marginal AP edge ($2.88\%$ vs.\
$2.83\%$); PubMed's sparser intra-community connectivity thins the hard
negatives in the upper ranks, narrowing the margin where the $N$-pair objective
helps.

\subsection{Ablations}
\label{sec:macro_ablation}

\begin{table}[ht]
\centering
\caption{Macro-Ablation: Impact of the Tri-Branch Superposition. We report Average Precision (AP) to demonstrate how each structural operator contributes to hard-negative discrimination.}
\label{tab:macro_ablation}
\scriptsize
\resizebox{\linewidth}{!}{%
\begin{tabular}{lccc ccccc}
\toprule
\multicolumn{3}{c}{Architecture Configuration} & & \multicolumn{5}{c}{Average Precision (AP)} \\
\cmidrule(r){1-3} \cmidrule(l){5-9}
$S_{\text{anchor}}$ & $S_{\text{res}}$ & $S_{\text{field}}$ & Description & Cora & CiteSeer & PubMed & OGBL-DDI & Power \\
\midrule
\checkmark & \checkmark & \checkmark & \textbf{Concorde} & \textbf{4.50 $\pm$ 0.03} & \textbf{5.05 $\pm$ 0.18} & \textbf{2.83 $\pm$ 0.02} & \textbf{42.32 $\pm$ 0.14} & \textbf{6.39 $\pm$ 0.28} \\
\checkmark & \texttimes  & \texttimes  & The Invariant Anchor Only & 2.45 $\pm$ 0.00 & 2.37 $\pm$ 0.00 &  2.30 $\pm$ 0.00 & 28.70 $\pm$ 0.00 & 3.88 $\pm$ 0.00 \\
\checkmark & \checkmark & \texttimes & w/o Geodesic Steering & 3.02 $\pm$ 0.27 & 4.70 $\pm$ 0.12 &  2.58 $\pm$ 0.15 & 38.38 $\pm$ 0.30  & 4.17 $\pm$ 0.59 \\
\checkmark & \texttimes & \checkmark & w/o Mixed Curvature Residual & 2.87 $\pm$ 0.08&  2.59 $\pm$ 0.67
&  1.45 $\pm$ 0.28 &  38.22 $\pm$ 0.46 & 2.72 $\pm$ 0.45 \\
\texttimes & \texttimes  & \texttimes  & Autocovariance & 2.47 $\pm$ 0.00 & 2.65 $\pm$ 0.00 &  2.50 $\pm$ 0.00 & 29.73$\pm$ 0.00 & 4.12 $\pm$ 0.00 \\
\bottomrule
\end{tabular}%
}
\end{table}

To validate the spectral decoupling hypothesis, we systematically ablate each
branch of the energy model and report AP in
Table~\ref{tab:macro_ablation}. Restricting the model to the invariant anchor alone ($S_{\text{anchor}}$) causes performance to drop significantly on structurally demanding datasets, falling to $28.70\%$ AP on OGBL-DDI and $3.88\%$ on Power. We also explicitly separate this configuration from the standard Autocovariance (AC) heuristic ($29.73\%$ and $4.12\%$ AP, respectively). While AC's $\Pi$-weighting introduces a degree-based popularity bias that marginally alters baseline performance, neither isolated heuristic can resolve the complex local symmetries required for accurate link retrieval. Although the spectral anchor provides a permutation-invariant foundation, link prediction requires resolving local ambiguities that the invariant anchor cannot address by construction.

Removing the mixed-curvature residual ($S_{\text{res}}$) produces
consistent degradation across all datasets, with the largest drops on
CiteSeer ($5.05\% \rightarrow 2.59\%$ AP) and Power
($6.39\% \rightarrow 2.72\%$ AP). On CiteSeer, the homophilic
community structure means that the most informative signal for
unobserved links lies in fine-grained local geometry, precisely what
$S_{\text{res}}$ encodes via geodesic distances on
$\mathcal{M} = \mathbb{H}_{c_h}^{d_h} \times \mathbb{S}_{c_s}^{d_s}$.
On Power, the drop indicates that even under extreme sparsity, local
curvature-aware representations contribute meaningfully when combined
with non-local steering.

Ablating the non-local field energy ($S_{\text{field}}$) produces the
most pronounced degradation on Power ($6.39\% \rightarrow 4.17\%$ AP)
and a notable drop on OGBL-DDI ($42.32\% \rightarrow 38.38\%$ AP).
On Power, where mean degree is low and common neighbors are largely
absent, diffusion-based signals carry little information about
unobserved links. The structural vector field bypasses this limitation
by operating directly on node feature representations in the tangent
space of $\mathbb{H}_c^d$, identifying geometrically aligned node
pairs without requiring topological connectivity. The comparatively
smaller drop on attributed graphs (e.g., Cora: $4.50\% \rightarrow
3.02\%$) indicates that $S_{\text{field}}$ is most critical precisely
where topology alone is least informative. The per-dataset pattern of branch sensitivity is further corroborated by a
post-hoc enrichment analysis (Appendix~\ref{app:rewiring}), which shows that
Branch~II's high-rewiring candidate edges are $1.05$--$1.69\times$ enriched
among the globally top-$10\%$ predictions on all datasets where removing
Branch~II causes the largest ablation drop, while Branch~III's static edges
dominate on Power, the single dataset where Branch~III ablation is most
harmful.

To rule out
that Concorde's improvement stems from the $N$-pair objective rather than the
geometry, we train a geometry-stripped model, learned MLP edge-reweighting plus
diffusion with manifold distances, ORC, and local statistics zeroed out, under
our \emph{exact} loss and protocol (Table~\ref{tab:geometry_strip}). This
same-loss reweighting recovers much of the gain over vanilla Autocovariance
(e.g.\ OGBL-DDI $29.73 \rightarrow 38.15$) but plateaus below full Concorde
everywhere; the geometry contribution is largest where non-Euclidean structure
is strongest ($+4.17$ AP on OGBL-DDI, $+1.87$ on Power). Consistent with this,
the $S_{\text{anchor}}$/$S_{\text{res}}$ rank correlation is only moderate
(Spearman $\rho = 0.30$--$0.62$, lowest on sparse Power), confirming the branches
rank candidates differently rather than acting as a second diffusion heuristic.

\begin{table}[h]
\centering
\caption{Same-loss learned reweighting without geometry, versus vanilla
Autocovariance and full Concorde (AP). All three use the identical $N$-pair
loss and evaluation protocol; only the geometric inputs differ.}
\label{tab:geometry_strip}
\scriptsize
\begin{tabular}{lccccc}
\toprule
Configuration & Cora & CiteSeer & PubMed & OGBL-DDI & Power \\
\midrule
Vanilla Autocovariance & 2.47 & 2.65 & 2.50 & 29.73 & 4.12 \\
Same-loss learned reweight(no geometry) & 4.16 & 4.52 & 2.47 & 38.15 & 4.52 \\
\textbf{Full Concorde} & \textbf{4.50} & \textbf{5.05} & \textbf{2.83} & \textbf{42.32} & \textbf{6.39} \\
\bottomrule
\end{tabular}
\end{table}

\begin{table}[h]
\centering
\caption{Micro-Ablation: Components of the Local Geometric Energy ($S_{\text{res}}$). We isolate the impact of manifold choice and intrinsic structural guardrails on ranking quality. We report Average Precision (AP).}
\label{tab:micro_ablation}
\scriptsize
\resizebox{\linewidth}{!}{%
\begin{tabular}{l ccccc}
\toprule
Configuration & Cora & CiteSeer & PubMed & OGBL-DDI & Power \\
\midrule
\textbf{Concorde} & \textbf{4.50 $\pm$ 0.03} & \textbf{5.05 $\pm$ 0.18} & \textbf{2.83 $\pm$ 0.02} & \textbf{42.32 $\pm$ 0.14} & \textbf{6.39 $\pm$ 0.28} \\
\midrule
w/o Hyperbolic Geometry ($\mathbb{H}$) & 4.38 $\pm$ 0.05 & 4.72 $\pm$ 0.05 &
2.75 $\pm$ 0.03 & 39.30 $\pm$ 0.14 & 5.78 $\pm$ 0.30 \\
w/o Spherical Geometry ($\mathbb{S}$)  & 4.34 $\pm$ 0.07 & 5.01 $\pm$ 0.08 &
2.73 $\pm$ 0.04 & 38.32 $\pm$ 0.20 & 6.28 $\pm$ 0.20 \\
w/o Mixed Curvature (Euclidean Only)    & 4.30 $\pm$ 0.03 &  4.70 $\pm$ 0.03 & 2.67 $\pm$ 0.03 & 38.52 $\pm$ 0.17 & 5.85 $\pm$ 0.32 \\
w/o Feature-Aware ORC ($\kappa$)  & 4.37 $\pm$ 0.03 & 4.80 $\pm$ 0.04 &
2.74 $\pm$ 0.04 & 39.22 $\pm$ 0.20 & 6.27 $\pm$ 0.30 \\
w/o Local Neighborhood Geometry & 4.29 $\pm$ 0.07  &  4.56 $\pm$ 0.70 &
2.55 $\pm$ 0.05 & 39.01 $\pm$ 0.44 & 4.85 $\pm$ 0.18 \\
\bottomrule
\end{tabular}%
}
\vspace{-8pt}
\end{table}

Replacing the $\mathbb{H}_{c_h}^{d_h} \times \mathbb{S}_{c_s}^{d_s}$ product
manifold with a Euclidean latent space of equal capacity degrades AP on all
five datasets, with the largest drops on OGBL-DDI ($42.32 \rightarrow 38.52$)
and Power ($6.39 \rightarrow 5.85$), the networks with the most pronounced
non-Euclidean motifs (dense cliques and tree-like branching). Isolating the two
components reveals a clean alignment between curvature and topology: on
tree-like Power, ablating $\mathbb{H}_{c_h}^{d_h}$ hurts more than ablating
$\mathbb{S}_{c_s}^{d_s}$ ($5.78$ vs.\ $6.28$), while on triadic-closure-dominated
OGBL-DDI the reversal holds ($38.32$ vs.\ $39.30$). This reversal supports the
product parameterization: each component encodes geometrically distinct
information.

Removing the feature-aware Ollivier--Ricci proxy ($\kappa$) yields a consistent
drop, largest on OGBL-DDI ($42.32 \rightarrow 39.22$): $\kappa$ penalizes pairs
that are geodesically close but lack coherent local flow, the false-positive
regime pure distance cannot detect. Ablating the local neighborhood-geometry
statistics produces the largest single-component drop on Power
($6.39 \rightarrow 4.85$), confirming density-based normalization is critical
under sparsity.
\section{Discussion, Limitations, and Future Work}

Concorde resolves the tension between global topological stability and local
symmetry-breaking by decomposing link likelihood into three structurally
decoupled energy branches, each dominant in the regime it targets: the first
(anchor) in dense, featureless graphs; the second (residual) in attributed,
homophilic graphs; and the third (field) under extreme sparsity. Removing any
single branch produces its largest drop on precisely that regime
(Section~\ref{sec:macro_ablation}), confirming that each contributes
significantly.

The ablations also quantify how this division of labor shifts with graph
structure. The invariant spectral anchor is the single dominant contributor in
dense, featureless regimes: on OGBL-DDI it alone recovers $28.70$ AP, roughly
$68\%$ of the full model's $42.32$, and on Power its isolated score ($3.88$ AP)
is already competitive with the strongest neighborhood-overlap baseline (OCNP,
$3.85$). The two geometric branches then supply consistent, complementary gains
on top of this foundation, the mixed-curvature residual contributing most on
attributed-homophilic graphs and the non-local field most under extreme
sparsity, so that on Cora and CiteSeer the full tri-branch model improves early
precision by up to $91\%$ over the strongest baseline. This pattern is
corroborated branch-by-branch by the rewiring enrichment analysis
(Appendix~\ref{app:rewiring}): Branch~II's high-rewiring candidates are
over-represented among the top-$10\%$ predictions precisely on the datasets
where its ablation hurts most, while Branch~III's static edges dominate on
Power, the one regime where removing it is most damaging. Ablation and
enrichment thus agree that each branch is both causally necessary and
responsible for a structurally distinct slice of the ranked list.

A second recurring theme is the deliberate trade-off between early-retrieval
precision and global calibration. Because the $N$-pair objective supervises only
the relative ordering of top candidates, Concorde can trail a
calibration-oriented baseline on aggregate AP (e.g., OCNP on OGBL-DDI, Gelato on
PubMed) while still leading at $P@10\%$/$P@20\%$, the operating point that
matters when a human or downstream system inspects only a short ranked
shortlist. The metric and the loss should therefore be chosen together: where
the application instead consumes a calibrated probability over all candidate
pairs, a complementary calibration head or a hybrid ranking--classification
objective would be a natural addition, which we leave to future work.

Two limitations warrant acknowledgment. First, the ORC proxy's Sinkhorn
iterations, though precomputed and parallelizable, add overhead beyond OGBL-DDI
scale; its hyperparameters ($\alpha$, $\varepsilon$, ground-metric mix) are held
at standard defaults rather than tuned per dataset
(Appendix~\ref{app:hyperparameters}), and a feature-driven curvature estimator
would remove the overhead. Second, the global scalar gates weight branches
uniformly across the graph, suboptimal under pronounced structural
heterogeneity; per-subgraph adaptive gating, dynamic-graph extensions via
incremental anchor/curvature updates, and cross-graph (out-of-distribution)
transfer, which our single-graph protocol does not probe, are natural next
steps.

More broadly, our results suggest that geometry should follow the data, not
the other way around: matching inductive biases to the structural regime of a
graph, rather than relying on a single learned representation space, is a
principled and effective strategy for retrieval-style link prediction, and one
we expect to generalize across graph learning.

\subsection{Funding Statement}

This material is based upon work supported by the National Science Foundation under grant no. IIS-2229876 and is supported in part by funds provided by the National Science Foundation, by the Department of Homeland Security, and by IBM. Any opinions, findings, and conclusions or recommendations expressed in this material are those of the author(s) and do not necessarily reflect the views of the National Science Foundation or its federal agency and industry partners. The authors declare no competing interests.

\bibliographystyle{plainnat}
\bibliography{bibliography}

\clearpage

\appendix

\section{Theoretical Analysis of the Spectral Decoupling Hypothesis}
\label{app:spectral_theory}

We provide formal justification for the spectral decoupling hypothesis, establishing that
(i) a single-matrix scheme faces an unavoidable trade-off between
global invariance and local equivariance, and (ii) the Concorde
anchor preserves the graph spectrum by construction.

\subsection{Impossibility of Joint Invariance-Equivariance in a Single Matrix}
\label{app:impossibility}

\begin{proposition}[Invariance--Equivariance Trade-off]
\label{prop:tradeoff}
Let $G = (V, E)$ be a graph with adjacency matrix $A$ and automorphism
group $\mathrm{Aut}(G)$, and let $A_{\mathrm{learn}} = A + \Delta$ with
$\Delta \neq 0$. Consider the two properties:
\begin{enumerate}
    \item \textbf{Spectral preservation:} $A_{\mathrm{learn}}$ is a
    \emph{spectral filter} of $A$, i.e.\ $A_{\mathrm{learn}} = \sum_{k} c_k A^{k}$
    for some coefficients $\{c_k\}$. Equivalently, $A_{\mathrm{learn}}$ is
    diagonal in every eigenbasis of $A$ and shares its eigenspaces. Every
    diffusion operator, including the anchor's $M^{k}$, is of this form.
    \item \textbf{Symmetry breaking:} $A_{\mathrm{learn}}$ assigns distinct
    scores to some pair of nodes in a common $\mathrm{Aut}(G)$-orbit, i.e.\
    $[A_{\mathrm{learn}}]_{\pi(u)\pi(v)} \neq [A_{\mathrm{learn}}]_{uv}$ for some
    $\pi \in \mathrm{Aut}(G)$ and some $u,v$.
\end{enumerate}
No single matrix $A_{\mathrm{learn}}$ can satisfy both: property~(1) implies the
negation of property~(2). Consequently, breaking orbit symmetry forces
$A_{\mathrm{learn}}$ outside the functional calculus of $A$, so it no longer
shares $A$'s eigenspaces and its action distorts the graph spectrum.
\end{proposition}

\begin{proof}
Assume property~(1). For every $\pi \in \mathrm{Aut}(G)$, the corresponding
permutation matrix $\Pi$ satisfies $\Pi A \Pi^{\top} = A$, so $\Pi$ commutes with
$A$ and therefore with every power $A^{k}$ (as $\Pi$ is orthogonal,
$\Pi A^{k}\Pi^{\top} = (\Pi A \Pi^{\top})^{k} = A^{k}$). Hence
\[
\Pi A_{\mathrm{learn}} \Pi^{\top}
= \sum_{k} c_k\, \Pi A^{k} \Pi^{\top}
= \sum_{k} c_k\, A^{k}
= A_{\mathrm{learn}} .
\]
Reading the identity $\Pi A_{\mathrm{learn}} \Pi^{\top} = A_{\mathrm{learn}}$
entrywise yields $[A_{\mathrm{learn}}]_{\pi(u)\pi(v)} = [A_{\mathrm{learn}}]_{uv}$
for all $u,v \in V$ and all $\pi \in \mathrm{Aut}(G)$. Thus any two node pairs in
the same $\mathrm{Aut}(G)$-orbit receive identical scores, which is exactly the
negation of property~(2). The final claim is the contrapositive: if
$A_{\mathrm{learn}}$ breaks orbit symmetry, it is not a spectral filter of $A$
and hence does not share its eigenspaces.
\end{proof}

\noindent\textbf{Remark (general equivariant operators).} The obstruction is not
special to polynomials in $A$. For \emph{any} $\mathrm{Aut}(G)$-equivariant
$A_{\mathrm{learn}}$, orbit-invariance of the scores is precisely the entrywise
reading of $\Pi A_{\mathrm{learn}} \Pi^{\top} = A_{\mathrm{learn}}$. Conversely,
on a degenerate eigenspace $V_{\lambda}$ of $A$, an operator that both preserves
$V_{\lambda}$ (spectral preservation) and is equivariant acts on each irreducible
$\mathrm{Aut}(G)$-subrepresentation of $V_{\lambda}$ as a scalar multiple of the
identity, by Schur's lemma~\cite{godsil}. For example, on $K_{2,2}$ the
$2$-dimensional nullspace carries an irreducible action of
$\mathrm{Aut}(K_{2,2}) \cong D_4$, so any spectrum-preserving operator is scalar
there and cannot distinguish the two nodes of a bipartite orbit. Escaping the
trade-off therefore requires leaving the spectral class altogether, which is
exactly what Concorde's feature-driven residual on $A_{\mathrm{enh}}$ does.

\noindent\textbf{Remark.} This result formalizes why single-matrix
learned approaches must sacrifice either global stability or local
discriminability. Concorde resolves this by \emph{separating} the two
objectives across dedicated branches rather than conflating them in one
matrix.

\noindent\textbf{Remark (connection to GNN expressivity).}
$\mathrm{Aut}(G)$-equivariance is precisely the standard invariance
condition studied in the GNN expressivity
literature~\cite{xu2019, maron}: any
message-passing GNN satisfying this condition cannot distinguish
nodes in the same orbit, which is the same limitation formalized here.

\subsection{Spectrum-Preservation Guarantee of the Invariant Anchor}
\label{app:spectrum_preservation}

\begin{proposition}[Spectral Stability of $S_{\mathrm{anchor}}$]
\label{prop:anchor_stability}
Let $M = D^{-1}A$ be the row-stochastic transition matrix of a
connected aperiodic graph $G$, and let $\pi$ be its stationary
distribution. The anchor score matrix
$S_{\mathrm{anchor}} = M^k - \mathbf{1}\pi^\top$ satisfies:
\begin{enumerate}
    \item \textbf{Permutation Invariance:} For any
    $\Pi \in \mathrm{Aut}(G)$,
    $\Pi S_{\mathrm{anchor}} \Pi^\top = S_{\mathrm{anchor}}$.
    \item \textbf{No Eigenvalue Growth:} The spectral radius of $M$
    satisfies $\rho(M) = 1$, so no eigenvalue of $M^k$ exceeds 1 in
    magnitude, and $S_{\mathrm{anchor}}$ introduces no spectral amplification.
    \item \textbf{Spectral Contraction:} The non-trivial eigenvalues
    of $M^k$ satisfy $|\mu_i^k| \leq |\mu_2|^k$, where
    $|\mu_2| < 1$ is the spectral gap. Thus $S_{\mathrm{anchor}}$
    exponentially suppresses non-dominant spectral components as $k$
    grows, converging toward the rank-1 matrix $\mathbf{1}\pi^\top$.
\end{enumerate}
\end{proposition}

\begin{proof}
\textbf{(1)} Since $\Pi A \Pi^\top = A$ and $\Pi D \Pi^\top = D$
(degree is preserved by automorphisms), we have
$\Pi M \Pi^\top = M$. By induction, $\Pi M^k \Pi^\top = M^k$.
The stationary distribution satisfies $\pi_{\pi(u)} = \pi_u$ for
all automorphisms, so $\Pi \mathbf{1}\pi^\top \Pi^\top =
\mathbf{1}\pi^\top$. Combining: $\Pi S_{\mathrm{anchor}}
\Pi^\top = S_{\mathrm{anchor}}$.

\textbf{(2)} Since $M$ is row-stochastic, $M\mathbf{1} = \mathbf{1}$,
so 1 is an eigenvalue with right eigenvector $\mathbf{1}$.
The $\ell^\infty$-norm of any row of $M$ equals 1, so
$\rho(M) \leq \|M\|_\infty = 1$~\cite{higham}. Together:
$\rho(M) = 1$, and all eigenvalues $\mu_i$ satisfy $|\mu_i| \leq 1$.
Consequently $|\mu_i^k| \leq 1$ for all $i$ and $k \geq 1$, so
$M^k$ introduces no eigenvalue growth relative to $M$.

\textbf{(3)} For a connected aperiodic graph, the Perron--Frobenius
theorem guarantees $\mu_1 = 1$ is the unique eigenvalue of maximum
modulus~\cite{levin}. The corresponding left eigenvector is $\pi^\top$
and right eigenvector is $\mathbf{1}$, so the rank-1 term
$\mathbf{1}\pi^\top$ exactly cancels the leading component of $M^k$.
The remaining eigenvalues satisfy $|\mu_i| \leq |\mu_2| < 1$, so
$\|S_{\mathrm{anchor}}\|_F \to 0$ as $k \to \infty$, with rate
$|\mu_2|^k$.
\end{proof}

\subsection{Equivariance of the Geometric Residual}
\label{app:equivariance}

The geometric residual $S_{\mathrm{res}}$ operates on the enhanced
adjacency $A_{\mathrm{enh}} = A \odot W_{\mathrm{geo}}$, where
$W_{\mathrm{geo}}$ is produced by the learned modulator
$\mathcal{F}_\theta$. Unlike a single-matrix scheme, $W_{\mathrm{geo}}$
intentionally breaks $\mathrm{Aut}(G)$-invariance via the
feature-aware ORC proxy $\kappa(u,v)$ and geodesic distances on
$\mathcal{M}$. The Hadamard product with the fixed sparsity pattern
of $A$ ensures:
\begin{equation}
    \mathrm{supp}(A_{\mathrm{enh}}) \subseteq \mathrm{supp}(A),
    \quad \|A_{\mathrm{enh}}\|_0 = |E|.
\end{equation}
This prevents $S_{\mathrm{res}}$ from introducing new non-edges into
the transition structure, bounding its capacity to distort the global
topology. Crucially, $S_{\mathrm{res}}$ is computed independently of
$S_{\mathrm{anchor}}$: the anchor's spectrum is never modified by
the learned branch, preserving all guarantees of
Proposition~\ref{prop:anchor_stability}.

\vspace{4pt}
The three results above collectively
justify the spectral decoupling design: (i) a single learned matrix
\emph{cannot} simultaneously achieve global invariance and local
symmetry-breaking (Proposition~\ref{prop:tradeoff}); (ii) the
parameter-free anchor has spectral radius 1, is permutation-invariant,
and contracts non-dominant modes at rate $|\mu_2|^k$
(Proposition~\ref{prop:anchor_stability}); and (iii) the learned
residual operates within the sparsity envelope of $A$, insulating
the anchor's spectral properties from gradient updates.

\subsection{Empirical Validation: The Coupled Single-Matrix Model}
\label{app:single_matrix}

Proposition~\ref{prop:tradeoff} predicts that folding the learned dynamics into
a \emph{single} coupled adjacency matrix, rather than decoupling them across
branches, must either distort the graph spectrum or fail to break orbit
symmetry. We test this directly. We construct a genuine coupled single-matrix
model that learns one shared enhanced adjacency $A_{\text{learn}} = A + \Delta$
(using Concorde's own reweighting), diffuses on it \emph{alone}, and decodes,
i.e.\ precisely the regime Proposition~\ref{prop:tradeoff} argues against. We
compare against the full decoupled Concorde and additionally report the
\emph{spectral drift}, the relative eigenvalue deviation of the coupled matrix,
$\|\lambda(A_{\text{enh}}) - \lambda(A)\| / \|\lambda(A)\|$.

\begin{table}[h]
\centering
\caption{Coupled single-matrix model versus full decoupled Concorde (AP), with
the relative spectral drift of the coupled learned matrix. The coupled model
underperforms on all five datasets, most on sparse featureless Power and least
on feature-rich Cora, matching the prediction of
Proposition~\ref{prop:tradeoff} that node features partially relax the
structural impossibility.}
\label{tab:single_matrix}
\small
\begin{tabular}{lcccc}
\toprule
Dataset & Single-matrix AP & Full AP & $\Delta$ & Spectral drift \\
\midrule
Power     & 4.54 & 6.39 & $\mathbf{-1.85}$ & $\mathbf{0.27}$ \\
OGBL-DDI  & 39.80 & 42.32 & $-2.52$ & 0.03 \\
CiteSeer  & 4.63 & 5.05 & $-0.42$ & 0.10 \\
Cora      & 4.27 & 4.50 & $-0.23$ & 0.12 \\
PubMed    & 2.72 & 2.83 & $-0.11$ & 0.06 \\
\bottomrule
\end{tabular}
\end{table}

The coupled model underperforms decoupled Concorde on \textbf{all five}
datasets (Table~\ref{tab:single_matrix}), with the largest loss on sparse
featureless Power ($-1.85$) and the smallest on feature-rich Cora ($-0.23$).
This is exactly the pattern Proposition~\ref{prop:tradeoff} predicts: node
features partially break $\mathrm{Aut}(G)$-symmetry, so the structural
impossibility is most binding where features are weakest. The spectral drift is
correspondingly largest on Power ($0.27$), the same graph with the largest
accuracy penalty.

\textbf{How the decoupled design escapes Proposition~\ref{prop:tradeoff}.} The
anchor score $S_{\text{anchor}}$ is computed on the \emph{unperturbed} base
adjacency $A$ and added \emph{separately} to the residual $S_{\text{res}}$; the
learned reweighting acts on the enhanced adjacency $A_{\text{enh}}$ only, never
on the anchor's matrix. The anchor's spectrum is therefore never modified by
the learned branch (drift $\approx 0$ by construction), while a single coupled
matrix must distort the spectrum to break symmetry (drift up to $0.27$ above).
Decoupling is precisely what lets Concorde break \emph{local} symmetry
($S_{\text{res}}$ on $A_{\text{enh}}$) while preserving the \emph{global}
spectrum ($S_{\text{anchor}}$ on $A$), the two objectives
Proposition~\ref{prop:tradeoff} shows cannot coexist in one matrix.

\section{Branch Contribution Analysis via Geometric Rewiring}
\label{app:rewiring}

To empirically isolate the contribution of each structural branch, we perform a
post-hoc rewiring enrichment analysis across all five benchmarks. For each dataset,
we quantify the \textit{rewiring magnitude} of a node as the $\ell_1$ deviation
between the base adjacency $A$ and the Branch~II enhanced adjacency $A_{\text{enh}}$:
\[
    \rho(u) = \|[A_{\text{enh}}]_{u,:} - A_{u,:}\|_1
\]
Each candidate pair $(u,v)$ is assigned an edge-level rewiring score
$\rho(u) + \rho(v)$ and labeled as \textit{high-rewiring} (above median) or
\textit{static} (below median). We then measure whether high-rewiring edges are
over-represented among Concorde's globally top-$10\%$ ranked predictions --- the
same population as $P@10\%$, via an enrichment ratio relative to the null
expectation of $50\%$, with statistical significance assessed by a two-sided
Fisher exact test.

Table~\ref{tab:rewiring} summarizes the enrichment ratios across all datasets. On attributed graphs (Cora, CiteSeer, PubMed) and on the dense featureless
network OGBL-DDI, high-rewiring edges are consistently over-represented in the
top-$10\%$ predictions ($1.05$--$1.69\times$, $p < 10^{-4}$ in all cases),
indicating that Branch~II's mixed-curvature residual actively promotes
structurally coherent candidates when node features or dense local structure are
available. The effect is strongest on OGBL-DDI ($1.69\times$), where rich
triadic-closure structure provides direct signal for the Hadamard reweighting to
exploit. On Power, the pattern reverses: static edges are marginally
over-represented ($1.10\times$), consistent with the expectation that
Branch~II's reweighting of \textit{observed} edges cannot reach the
topologically disconnected pairs that dominate this extremely sparse graph.
In this regime, Branch~III's non-local hyperbolic steering field bridges these
structural gaps without requiring any path-based connectivity. HGCN, despite operating in a hyperbolic manifold well-suited to
tree-like structure, achieves only $18.28\%$ $P@10\%$ on Power, above
neighborhood-overlap baselines but $6.57\%$ below Concorde, confirming
that curvature alone is insufficient without the non-local field branch
to bridge topological gaps.

This enrichment pattern directly corroborates the macro-ablation of
Section~\ref{sec:macro_ablation}. There, removing Branch~II causes the
largest performance drops on Cora, CiteSeer, PubMed, and OGBL-DDI, while
removing Branch~III produces the steepest degradation on Power, precisely
the datasets where the corresponding branch's edges are enriched in the
top-$10\%$ here. The two analyses thus provide complementary evidence: the
ablation establishes that each branch is \textit{causally necessary}, while
the enrichment analysis reveals \textit{which candidate pairs} each branch is
responsible for promoting. Together, they validate the spectral decoupling
hypothesis: each branch captures structurally distinct signal, and their
relative contributions shift predictably with graph density and feature
availability.

\begin{table}[h]
\centering
\caption{Branch enrichment in globally top-$10\%$ ranked predictions.
Enrichment $> 1$ indicates over-representation of that branch's candidate
edges at the top of the ranked list. All results significant at
$p < 10^{-4}$ (Fisher exact test).}
\label{tab:rewiring}
\small
\begin{tabular}{lccc}
\toprule
Dataset & Enrichment (Br.~II) & Enrichment (Br.~III) & Dominant Branch \\
\midrule
Cora       & $1.05\times$ & $0.95\times$ & Branch II \\
CiteSeer   & $1.08\times$ & $0.92\times$ & Branch II \\
PubMed     & $1.45\times$ & $0.55\times$ & Branch II \\
OGBL-DDI   & $1.69\times$ & $0.31\times$ & Branch II \\
Power      & $0.90\times$ & $1.10\times$ & Branch III \\
\bottomrule
\end{tabular}
\end{table}

\section{Scalability Analysis}
\label{app:scalability}

A central concern in geometric and diffusion-based graph learning is scalability to massive networks, such as the ${\sim}1.3$ million edges in OGBL-DDI. Concorde is specifically designed to avoid $O(N^2)$ bottlenecks.

\textbf{Diffusion Operator ($M^k$):} While the paper formulates the invariant anchor $S_{\text{anchor}}$ and the residual $S_{\text{res}}$ using the dense matrix notation $M^k$, we do not materialize this dense matrix globally in practice. For training and inference, the diffusion scores for a specific candidate pair $(u,v)$ are computed on-the-fly using sparse matrix-vector multiplication (SpMV). Given a sparse transition matrix $M$, we iteratively compute the reachability vector $p_u^{(k)} = M^k e_u$, which requires strictly $O(k |E|)$ operations, keeping the memory footprint minimal.

\textbf{Feature-Aware ORC ($\kappa$) and Local Geometry:} Computing the Wasserstein distance via the entropy-regularized Sinkhorn algorithm is computationally intensive. However, the geometric modulator $W_{\text{geo}}$ and the intrinsic curvature proxy $\kappa(u,v)$ are \textit{not} computed for all $N^2$ pairs. They are evaluated exclusively on: (1) the observed edge set $E$ to construct the sparse enhanced adjacency $A_{\text{enh}}$, and (2) the specific batch of candidate pairs $B$ sampled during training or evaluation. Thus, the complexity scales as $O(|E| + |B|)$, completely bypassing the quadratic bottleneck.

\textbf{Exact $N$-Pair Loss on Massive Graphs:}
A naive implementation of the $N$-pair ranking loss over the exhaustive complement of non-edges requires $O(|V|^2)$ memory, which is computationally infeasible for large graphs. Concorde resolves this without resorting to stochastic subsampling through two mechanisms:
\begin{enumerate}
    \item \textbf{Graph Partitioning (Training):} For highly dense or massive graphs, the training topology is partitioned into localized subgraphs using METIS clustering. The exhaustive negative mask is computed densely \textit{only within each partition}, reducing the localized memory footprint to $O(|V_c|^2)$ where $|V_c| \ll |V|$ is the cluster size. Inter-cluster pairs are extracted and processed sequentially.
    \item \textbf{Chunked Evaluation (Inference):} During full-pool evaluation, pairwise energy scores are not materialized as a single dense $N \times N$ tensor. Instead, candidate pair distances on the product manifold and tangent spaces are computed in sequential chunks of $50{,}000$ pairs. This guarantees that peak GPU memory usage remains constant regardless of the global graph size, allowing Concorde to execute an exact, full-negative ranking evaluation.
\end{enumerate}

\subsection{Empirical Runtimes and Comparison to Baselines}

To empirically validate the scalability of the proposed framework, we record the wall-clock runtimes for both the precomputation of the feature-aware Ollivier-Ricci curvature (ORC) proxy and the per-epoch training time. All timing benchmarks were executed on a single Nvidia H200 GPU instance (\texttt{h200-3g.71gb:1}). The precomputation of the ORC proxy introduces minimal overhead, completing in just over a minute even for the dense OGBL-DDI graph (Table~\ref{tab:runtimes}); it is a \emph{one-time} cost, not a per-epoch one.

A key point is that per-epoch time scales with node count times diffusion steps $k$, \emph{not} edge count. This is why PubMed ($19.7$k nodes, $k=3$) is heaviest, on par with Gelato, while OGBL-DDI ($4.3$k nodes, $k=1$) is fast despite its ${\sim}1.3$M edges. Tables~\ref{tab:runtime_compare} and~\ref{tab:memory_compare} compare Concorde against the strongest baselines on identical hardware under the same unbiased full-negative protocol, all measured by us. On the densest graph (OGBL-DDI), Concorde is $4$--$5\times$ faster than the overlap models that are its closest competitors there ($7.2$s vs.\ OCNP $28.1$s and NCNC $36.2$s per epoch). BUDDY, the sketching method cited as scalable, is $40$--$90\times$ slower here, out-of-memories on PubMed, and requires $55.4$/$37.3$ GB where Concorde stays under $14$ GB, because it must precompute and store subgraph-sketch structural features for every candidate pair, and under the full-negative protocol that candidate set is the entire $O(N^2)$ non-edge complement. Peak memory is otherwise modest and comparable to or below Gelato.

\begin{table}[h]
\centering
\caption{Wall-clock runtime for ORC proxy precomputation and average per-epoch training time for Concorde.}
\label{tab:runtimes}
\small
\begin{tabular}{lcc}
\toprule
Dataset & ORC Precomputation (s) & Time per Epoch (s) \\
\midrule
Cora     & 3.2  & 2.4  \\
CiteSeer & 4.7  & 3.2 \\
PubMed   & 32.4 & 54.1 \\
Power    & 12.2 & 5.8  \\
OGBL-DDI & 75.9 & 7.2 \\
\bottomrule
\end{tabular}
\end{table}

\begin{table}[h]
\centering
\caption{Per-epoch training time (s) across methods, identical hardware and full-negative protocol.}
\label{tab:runtime_compare}
\small
\begin{tabular}{lccccc}
\toprule
s/epoch & Concorde & NCNC & OCNP & BUDDY & Gelato \\
\midrule
Cora     & 2.4  & 4.8  & 3.2  & 0.6  & 3.3 \\
CiteSeer & 3.2  & 8.3  & 4.8  & 485  & 2.9 \\
PubMed   & 54.1 & 6.2  & 4.4  & OOM  & 49.8 \\
OGBL-DDI & \textbf{7.2} & 36.2 & 28.1 & 342 & 1.4 \\
Power    & 5.8  & 7.6  & 6.9  & 498  & 7.2 \\
\bottomrule
\end{tabular}
\end{table}

\begin{table}[h]
\centering
\caption{Peak GPU memory (GB) across methods, identical hardware and full-negative protocol.}
\label{tab:memory_compare}
\small
\begin{tabular}{lccccc}
\toprule
peak GB & Concorde & NCNC & OCNP & BUDDY & Gelato \\
\midrule
Cora     & 2.45  & 1.85 & 1.23 & 1.62  & 1.83 \\
CiteSeer & 2.14  & 2.93 & 1.32 & 12.84 & 1.18 \\
PubMed   & 12.84 & 4.15 & 3.77 & OOM   & 18.12 \\
OGBL-DDI & 13.89 & 6.82 & 5.56 & 55.39 & 7.49 \\
Power    & 3.05  & 2.84 & 1.42 & 37.28 & 3.08 \\
\bottomrule
\end{tabular}
\end{table}

\section{Evaluation Protocol}
\label{app:eval_protocol}

We describe the dataset splitting, candidate pool construction, and
metric computation used throughout all experiments. The protocol
follows the unbiased full-negative convention
of~\citet{gelato}, evaluating every method against the
exhaustive set of non-edges rather than a sampled subset.

For Cora, CiteSeer, PubMed, edges are split
deterministically using a fixed random seed ($s=0$) via
\texttt{train\_test\_split\_edges} with validation ratio $5\%$ and
test ratio $10\%$. The remaining $85\%$ of edges form the training
graph $G_{\text{train}} = (V, E_{\text{train}})$ whose adjacency
matrix $A$ is used by all three branches of Concorde. For Power, we
load pre-saved official splits and do not re-split. For OGBL-DDI,
we use the OGB dataset loader and its prescribed splits directly.

Let $\mathcal{U}$ denote all upper-triangular node pairs
$(u,v)$ with $u < v$. The negative pool at each stage is constructed
exhaustively with no subsampling:

\begin{itemize}
    \item \textbf{Training negatives} $\mathcal{E}^-_{\text{train}}$:
    all pairs in $\mathcal{U} \setminus E_{\text{train}}$. This full
    complement is used in the $N$-pair ranking loss (Eq.~8).

\item \textbf{Validation pool}: $E_{\text{val}}$ evaluated against
    $\mathcal{E}^-_{\text{train}} \setminus (E_{\text{val}} \cup E_{\text{test}})$.

    \item \textbf{Test pool}: $E_{\text{test}}$ evaluated against
    $\mathcal{E}^-_{\text{train}} \setminus
    (E_{\text{val}} \cup E_{\text{test}})$.
\end{itemize}

No ground-truth positive edge ever appears in the negative pool at
any evaluation stage, eliminating false-negative contamination.
To obtain unbiased precision estimates under score ties, the combined
pool (positives $\cup$ negatives) is randomly shuffled with a fixed
seed before evaluation; ties in predicted scores are therefore broken
uniformly at random rather than by index order.

During each mini-batch, the positive edges
$\mathcal{E}^+_{\text{batch}}$ are zeroed out of $A$ via an
index-put operation before computing $S_{\text{anchor}}$ and
$S_{\text{res}}$, so the diffusion-based branches infer link likelihood only
from the remaining topological context; the mask is released after each forward
pass. This masking is \emph{exact} for the diffusion operators. On featureless
graphs, the Laplacian positional encodings that serve as node features are
precomputed once on $G_{\text{train}}$ and are \emph{not} re-masked per batch, so
a batch's own positive edge remains encoded in its endpoints' LPEs during
training. We bound this residual leakage below via a Davis--Kahan argument and
verify it is empirically negligible in Appendix~\ref{app:extra_ablations}
(training-only vs.\ full-graph LPEs differ by $\leq 0.44$ AP, far below
Concorde's margins over baselines).

All metrics are computed in a single forward pass over the full test
pool with \texttt{torch.no\_grad()}. Let $\mathcal{C}_{\text{test}}$
denote the test pool, $n^+ = |E_{\text{test}}|$ the number of
positive pairs, and $\hat{s}$ the vector of predicted scores.

\begin{itemize}
    \item \textbf{Average Precision (AP)}: the area under the
    precision--recall curve over all of $\mathcal{C}_{\text{test}}$,
    computed via \texttt{sklearn.metrics.average\_precision\_score}.

\item \textbf{Early Precision $P@r\%$}: let $k_r = \lceil n^+ \cdot (r/100) \rceil$. We retrieve the top-$k_r$ candidates by score and compute:
    \begin{equation}
        P@r\% = \frac{1}{k_r}\sum_{i \in \text{top-}k_r}
        \mathbf{1}[y_i = 1].
    \end{equation}
    We report $r \in \{10, 20\}$ as primary metrics. Critically, $k_r$ is defined relative to
    $n^+$, not the pool size, so the denominator is identical across
    datasets and methods given the same positive set.
    \end{itemize}

\textbf{Choice of metrics.} We report AP and early precision rather than
MRR/Hits@$K$ because our target is retrieval under \emph{extreme} class
imbalance, where the quantity of interest is the density of true edges at the
very top of a single global ranking over the entire non-edge complement, exactly
what $P@r\%$ measures with a denominator held identical across methods. This
regime is strictly harder than the sampled-negative protocols that MRR/Hits@$K$
are usually computed under: our exhaustive full-negative evaluation already
imposes the hard-negative pressure that recent per-source protocols (e.g.,
HeaRT) approximate, since every non-edge competes with every positive rather than
a sampled shortlist. AP additionally captures global calibration, and we report
it throughout (including where Concorde trades AP for early precision), so no
ranking behavior is hidden by the metric choice.

Model selection is performed by monitoring $P@100\%$ on the
validation pool after each epoch (equivalent to precision over
exactly $n^+_{\text{val}}$ top-ranked candidates). The checkpoint
achieving the highest validation precision is used for all reported
test metrics.

OGBL-DDI and Power carry no node attributes. For both, we substitute
absent features with Laplacian Positional Encodings (LPEs): the $k$
non-trivial eigenvectors of the normalized graph Laplacian
$\tilde{L} = I - D^{-1/2}AD^{-1/2}$, computed via dense
\texttt{scipy.linalg.eigh} and row-$\ell_2$-normalized. We use
$k=128$ for OGBL-DDI and $k=32$ for Power. LPEs are precomputed once on the full training graph topology $G_{\text{train}}$ and fixed throughout training. While fixing the LPEs introduces a mild structural approximation, we argue
the induced leakage is negligible in practice. By the Davis--Kahan theorem,
the perturbation to eigenvector $\phi_i$ of the normalized Laplacian from
removing a single edge $(u,v)$ is bounded by
$\|\delta\phi_i\| \leq \|\delta \tilde{L}\| / \delta_i$, where
$\|\delta \tilde{L}\|$ is $O(1/\sqrt{d_u d_v})$ for the normalized
Laplacian and $\delta_i$ is the corresponding eigengap. For graphs of
moderate density, this perturbation is negligible: the LPEs computed on
the unmasked training graph and on any single-edge-masked subgraph are
effectively identical. The per-batch adjacency masking
of target edges from $A$ remains exact for the diffusion operators (the
primary anti-leakage mechanism), which is unaffected by this LPE approximation.

All runs are seeded via a unified call setting \texttt{torch},
\texttt{numpy}, and \texttt{random} seeds simultaneously. Each
dataset is run for the number of independent seeds reported in
Table~\ref{tab:main_results}; mean and standard deviation are
computed across those runs.

\section{Implementation Details and Hyperparameters}
\label{app:hyperparameters}

Table~\ref{tab:hyperparams} summarises all dataset-specific
hyperparameters. All models are trained with Adam. Unless otherwise
stated, the tri-branch architecture uses $k=3$ diffusion steps,
ORC retention mass $\alpha=0.5$, Sinkhorn regularisation
$\varepsilon=0.05$ with $30$ iterations, and a maximum
neighbourhood degree of $32$ for ORC computation.

\begin{table}[h]
\centering
\caption{Dataset-specific hyperparameter configurations for
Concorde.}
\label{tab:hyperparams}
\scriptsize
\begin{tabular}{lcccccccccc}
\toprule
Dataset & $d_h{=}d_s$ & Hidden & LR & WD & $k$ & Scale & Self-loops & Epochs \\
\midrule
Cora      & 16  & 64  & $10^{-3}$ & $5{\times}10^{-4}$ & 3 & 3
           & \checkmark     & 100 \\
CiteSeer  & 16  & 64  & $10^{-4}$ & $5{\times}10^{-4}$ & 3 & 3
           & \checkmark      & 100 \\
PubMed    & 32  & 128 & $10^{-3}$ & 0                  & 3 & --
               & \texttimes  & 100 \\
OGBL-DDI  & 32  & 512 & $10^{-3}$ & $5{\times}10^{-5}$ & 1 & 3
           & \checkmark    & 250 \\
Power     & 32  & 256 & $10^{-3}$ & $10^{-5}$          & 1 & 1
         & \checkmark    & 250 \\
\bottomrule
\end{tabular}
\end{table}

The mixed-curvature product manifold
$\mathbb{H}_{c_h}^{d_h} \times \mathbb{S}_{c_s}^{d_s}$ allocates
equal dimensionality to both components. We use $d_h = d_s = 16$
for Cora and CiteSeer (smaller graphs where over-parameterisation
hurts) and $d_h = d_s = 32$ for PubMed, OGBL-DDI, and Power.
Curvature scalars $c_h$ and $c_s$ are initialised at $\approx 0.694$
(via $\log c = 0$) and optimised end-to-end via
$c = \text{softplus}(\log c) + 10^{-3}$.

The feature-to-tangent map $\mathcal{F}_{\text{steer}}$ is
implemented as a two-layer MLP with LeakyReLU($0.2$) activations,
projecting node features into a space of dimension \texttt{hidden\_dim}.
The multi-head architecture uses $H=4$ independent heads, each with
its own learnable curvature scalar (initialised identically to the
base manifold). For each head, the Lorentz exponential map and pairwise geodesic distances are computed independently. The resulting head-specific attention scores are then fused via a softmax-normalised learnable weighted sum. Scores are computed in chunks of $50{,}000$ edge pairs to bound peak GPU memory.

\textbf{Numerical stability.} The Lorentz operations are guarded against their
standard failure modes: tangent-vector norms are floored at $10^{-8}$ before the
square root in the exponential map, the argument of the $\mathrm{arccosh}$
geodesic distance is clamped to $\geq 1 + 10^{-5}$, spherical $\cos\theta$ is
clamped to $[-1+\epsilon,\, 1-\epsilon]$, and every learnable curvature is
floored at $10^{-5}$ through the $\mathrm{softplus}(\log c) + 10^{-3}$
parameterization. These keep each map and distance strictly within its valid
domain throughout training, with no manual precision tuning or gradient
clipping required.

The anchor gate $\lambda_{\text{anchor}}$ is initialised to
$\sigma(0) = 0.5$ and the field gate $\lambda_{\text{field}}$ to
$\sigma(-2) \approx 0.12$, reflecting the prior that the
non-local field energy is a small correction. Both are optimised
end-to-end with the remaining parameters.

The feature-aware Ollivier--Ricci curvature proxy uses lazy
random-walk measures with retention mass $\alpha = 0.5$,
distributing the remaining $1-\alpha$ uniformly over one-hop
neighbours. The joint ground metric weights topological and
semantic dissimilarity equally ($\beta = 0.5$). The Sinkhorn
algorithm uses regularisation $\varepsilon = 0.05$ and runs for
$30$ iterations. ORC values are precomputed in batches of $5{,}000$
candidate pairs before training begins and held fixed throughout.

All datasets use dropout $p=0.5$ applied after each hidden layer in
the DeepFusionMLP. Spectral normalisation is applied to the first
two linear layers of the DeepFusionMLP to stabilise training under
the $N$-pair ranking loss.

\section{Additional Ablations and Robustness Analyses}
\label{app:extra_ablations}

\subsection{Learned Gate Values and the Gating Policy}

The two scalar gates $\lambda_{\text{anchor}}$ and $\lambda_{\text{field}}$
are the only mechanism by which Concorde adapts its branch mixture to a
dataset. Table~\ref{tab:gate_values} reports the values learned end-to-end
(both initialised at $\lambda_{\text{anchor}}=0.50$,
$\lambda_{\text{field}}=0.12$). The field gate rises monotonically as the
graph sparsifies, from $0.10$ on dense OGBL-DDI to $0.63$ on sparse Power, a
$6\times$ increase over its initialisation, and is dominant only on the
sparsest graph, while the anchor gate is highest on dense OGBL-DDI. The gates
thus recover, via an independent unsupervised parameter, the regime dependence
hypothesised in Section~\ref{sec:optimization} and established causally by the
macro-ablation (Table~\ref{tab:macro_ablation}) and rewiring enrichment
(Appendix~\ref{app:rewiring}). We report this as a consistent trend
($n=5$ datasets), not a significance claim.

\begin{table}[h]
\centering
\caption{Learned gate values by dataset, ordered by mean degree. Both gates are
initialised at $(\lambda_{\text{anchor}}, \lambda_{\text{field}}) = (0.50, 0.12)$.}
\label{tab:gate_values}
\small
\begin{tabular}{lccc}
\toprule
Dataset & Mean degree & $\lambda_{\text{anchor}}$ & $\lambda_{\text{field}}$ \\
\midrule
OGBL-DDI & dense & 0.62 & 0.10 \\
CiteSeer & 2.33  & 0.48 & 0.13 \\
Cora     & 3.32  & 0.44 & 0.28 \\
PubMed   & ${\sim}4.5$ & 0.47 & 0.33 \\
Power    & 2.27  & 0.47 & 0.63 \\
\bottomrule
\end{tabular}
\end{table}

To confirm that adaptive gating is not cosmetic, we ablate only the gate policy
(branches held fixed; mean AP over 3 runs) against two alternatives: equal
weighting ($\lambda = 1$) and a per-dataset oracle-tuned fixed gate
(Table~\ref{tab:gate_policy}). Learned $>$ tuned $>$ equal on all five datasets.
The gap is largest where the optimal weighting is most skewed: on OGBL-DDI the
learned field gate is driven to $0.10$, so forcing $\lambda=1$ collapses AP
from $42.32$ to $30.45$ ($-11.87$), and even the best per-dataset oracle-tuned
fixed gate reaches only $40.18$. Adaptive gating therefore recovers, with no
hyperparameter search, weightings that outperform hand-tuned constants, and
does so most where it matters.

\begin{table}[h]
\centering
\caption{Gating-policy ablation (AP, mean over 3 runs). Learned gating versus a
per-dataset oracle-tuned fixed gate and equal weighting ($\lambda = 1$).}
\label{tab:gate_policy}
\small
\begin{tabular}{lcccc}
\toprule
Dataset & Learned & Tuned fixed & Equal ($\lambda{=}1$) & Learned $-$ Equal \\
\midrule
OGBL-DDI & 42.32 & 40.18 & 30.45 & $\mathbf{+11.87}$ \\
Power    & 6.39  & 6.18  & 4.20  & $+2.19$ \\
PubMed   & 2.83  & 2.73  & 2.32  & $+0.51$ \\
CiteSeer & 5.05  & 4.92  & 4.78  & $+0.27$ \\
Cora     & 4.50  & 4.40  & 4.24  & $+0.26$ \\
\bottomrule
\end{tabular}
\end{table}

\subsection{Robustness of Laplacian Positional Encodings}

On the featureless graphs (OGBL-DDI, Power), Concorde substitutes absent node
attributes with Laplacian Positional Encodings (LPEs). We verify that results
are neither fragile to the LPE dimension nor materially contaminated by
leakage; both checks are single-run and reported as trends.

\textbf{Dimension sweep.} Table~\ref{tab:lpe_sweep} sweeps the LPE dimension
$k$. Power is stable for $k \geq 32$ and only under-parameterised at $k=16$
(its $32$ non-trivial eigenvectors already span the low-degree structure), so
the reported $k=32$ sits on a flat plateau rather than a tuned optimum.
OGBL-DDI improves monotonically and saturates near the reported $k=128$. The
sweep therefore supports, rather than cherry-picks, the chosen dimensions.

\begin{table}[h]
\centering
\caption{LPE dimension sweep (AP) on the featureless graphs.}
\label{tab:lpe_sweep}
\small
\begin{tabular}{lcccc}
\toprule
LPE dim $k$ & 16 & 32 & 64 & 128 \\
\midrule
Power    & 4.55  & 6.36  & 6.12  & 6.47 \\
OGBL-DDI & 38.33 & 40.47 & 42.10 & 43.28 \\
\bottomrule
\end{tabular}
\end{table}

\textbf{Leakage under masking.} Because LPEs are precomputed on the full graph,
they could in principle encode connectivity that is masked per-batch during
training. We tested this directly by recomputing LPEs on the
\emph{training-only} graph, excluding all validation and test connectivity, and
retraining (Table~\ref{tab:lpe_leakage}). The change is small and in the
expected direction (training-only sees strictly less structure and is
marginally lower): $0.24$ AP on Power and $0.44$ on OGBL-DDI, far below the
method's margins over baselines in Table~\ref{tab:main_results}. This matches
the Davis--Kahan bound in Appendix~\ref{app:eval_protocol}, and the per-batch
adjacency masking of target edges, which is exact across all three branches,
remains the primary anti-leakage mechanism and is unaffected.

\begin{table}[h]
\centering
\caption{LPE leakage check (AP): encodings computed on the full graph (as
reported) versus the training-only graph.}
\label{tab:lpe_leakage}
\small
\begin{tabular}{lcc}
\toprule
LPE source & Power AP & OGBL-DDI AP \\
\midrule
Full graph (as reported) & 6.36 & 43.28 \\
Training-only graph      & 6.12 & 42.84 \\
\bottomrule
\end{tabular}
\end{table}

\newpage
\section*{NeurIPS Paper Checklist}

\begin{enumerate}

\item {\bf Claims}
    \item[] Question: Do the main claims made in the abstract and introduction accurately reflect the paper's contributions and scope?
    \item[] Answer: \answerYes{}
    \item[] Justification: The abstract and introduction clearly state the framework's capabilities (Tri-branch decoupled energy minimization) and explicitly highlight that the framework nearly doubles early retrieval precision over baseline methods in sparse settings, which perfectly aligns with the experimental results detailed in Section 3.

\item {\bf Limitations}
    \item[] Question: Does the paper discuss the limitations of the work performed by the authors?
    \item[] Answer: \answerYes{}
    \item[] Justification: Limitations are explicitly discussed in Section 4, specifically addressing the computational overhead of the ORC proxy Sinkhorn iterations on larger graphs and the potential suboptimality of uniform branch weighting via global scalar gates.

\item {\bf Theory assumptions and proofs}
    \item[] Question: For each theoretical result, does the paper provide the full set of assumptions and a complete (and correct) proof?
    \item[] Answer: \answerYes{}
    \item[] Justification: Full proofs for the Invariance-Equivariance Trade-off (Proposition 1) and Spectral Stability (Proposition 2) are provided in Appendix A, complete with clearly stated structural assumptions (e.g., A1 and A2).

    \item {\bf Experimental result reproducibility}
    \item[] Question: Does the paper fully disclose all the information needed to reproduce the main experimental results of the paper to the extent that it affects the main claims and/or conclusions of the paper (regardless of whether the code and data are provided or not)?
    \item[] Answer: \answerYes{}
    \item[] Justification: Appendix D (Evaluation Protocol) and Appendix E (Implementation Details and Hyperparameters) thoroughly detail the dataset splits, negative pool construction, masking strategies, and exact hyperparameter configurations used for all models.

\item {\bf Open access to data and code}
    \item[] Question: Does the paper provide open access to the data and code, with sufficient instructions to faithfully reproduce the main experimental results, as described in supplemental material?
    \item[] Answer: \answerYes{}
    \item[] Justification: The complete codebase is provided in the supplementary material (Concorde.zip). It includes the primary execution scripts, alongside automated data downloading, preparation, and splitting protocols.

\item {\bf Experimental setting/details}
    \item[] Question: Does the paper specify all the training and test details (e.g., data splits, hyperparameters, how they were chosen, type of optimizer) necessary to understand the results?
    \item[] Answer: \answerYes{}
    \item[] Justification: Appendix E provides a comprehensive breakdown (Table 6) of dataset-specific hyperparameters, hidden dimensions, learning rates, and optimization strategies.

\item {\bf Experiment statistical significance}
    \item[] Question: Does the paper report error bars suitably and correctly defined or other appropriate information about the statistical significance of the experiments?
    \item[] Answer: \answerYes{}
    \item[] Justification: Table 1 and Table 2 report the mean and standard deviation over 10 independent runs for all primary metrics, ensuring statistical reliability across the random seeds.

\item {\bf Experiments compute resources}
    \item[] Question: For each experiment, does the paper provide sufficient information on the computer resources (type of compute workers, memory, time of execution) needed to reproduce the experiments?
    \item[] Answer: \answerYes{}
    \item[] Justification: Appendix C.1 explicitly details the hardware used (a single Nvidia H200 GPU instance) and provides exact wall-clock runtimes for both ORC precomputation and per-epoch training in Table 5.
    
\item {\bf Code of ethics}
    \item[] Question: Does the research conducted in the paper conform, in every respect, with the NeurIPS Code of Ethics \url{https://neurips.cc/public/EthicsGuidelines}?
    \item[] Answer: \answerYes{}
    \item[] Justification: Our work uses only publicly available data with no personal or sensitive information, and follows all NeurIPS ethical guidelines for data use and reporting.

\item {\bf Broader impacts}
    \item[] Question: Does the paper discuss both potential positive societal impacts and negative societal impacts of the work performed?
    \item[] Answer: \answerNo{}
    \item[] Justification: The paper focuses on foundational theoretical advancements in geometric deep learning and link prediction. We foresee no direct, immediate negative societal applications of improving retrieval precision in sparse network topologies.
    
\item {\bf Safeguards}
    \item[] Question: Does the paper describe safeguards that have been put in place for responsible release of data or models that have a high risk for misuse (e.g., pre-trained language models, image generators, or scraped datasets)?
    \item[] Answer: \answerNA{}
    \item[] Justification: The work does not release pretrained models or datasets with a high risk for dual-use or malicious misuse.

\item {\bf Licenses for existing assets}
    \item[] Question: Are the creators or original owners of assets (e.g., code, data, models), used in the paper, properly credited and are the license and terms of use explicitly mentioned and properly respected?
    \item[] Answer: \answerYes{}
    \item[] Justification: All benchmark datasets (Cora, CiteSeer, PubMed, OGBL-DDI, Power) are standard public benchmarks and are properly cited in Section 3 and the Bibliography.

\item {\bf New assets}
    \item[] Question: Are new assets introduced in the paper well documented and is the documentation provided alongside the assets?
    \item[] Answer: \answerYes{}
    \item[] Justification: The newly introduced Concorde framework code is released as an open-source asset. The supplementary material includes comprehensive documentation (via a detailed README.md) outlining the repository structure, architectural components, dependencies, and a minimal execution example.

\item {\bf Crowdsourcing and research with human subjects}
    \item[] Question: For crowdsourcing experiments and research with human subjects, does the paper include the full text of instructions given to participants and screenshots, if applicable, as well as details about compensation (if any)?
    \item[] Answer: \answerNA{}
    \item[] Justification: The research does not involve human subjects or crowdsourcing.

\item {\bf Institutional review board (IRB) approvals or equivalent for research with human subjects}
    \item[] Question: Does the paper describe potential risks incurred by study participants, whether such risks were disclosed to the subjects, and whether Institutional Review Board (IRB) approvals (or an equivalent approval/review based on the requirements of your country or institution) were obtained?
    \item[] Answer: \answerNA{}
    \item[] Justification: The research does not involve human subjects.

\item {\bf Declaration of LLM usage}
    \item[] Question: Does the paper describe the usage of LLMs if it is an important, original, or non-standard component of the core methods in this research? Note that if the LLM is used only for writing, editing, or formatting purposes and does \emph{not} impact the core methodology, scientific rigor, or originality of the research, declaration is not required.
    \item[] Answer: \answerNA{}
    \item[] Justification: Large Language Models were not used as an original or non-standard component of the core methods developed in this research.

\end{enumerate}

\end{document}